\documentclass[final]{siamart1116}
\DeclareGraphicsExtensions{.pdf,.png}
\usepackage{multirow}
\usepackage{array}
\usepackage{relsize}
\usepackage{amsmath, amssymb, url, doi, xspace}
\usepackage{subfigure}
\usepackage{relsize}
\usepackage{mathtools}
\usepackage[linesnumbered,algoruled,boxed,lined,algo2e]{algorithm2e}
\usepackage{tikz}
\usetikzlibrary{patterns}
\usetikzlibrary{arrows, decorations.markings}
\usepackage{ifpdf}
\usepackage{epstopdf}
\usepackage{color}

\newcommand{\ie}{i.e.\xspace}

\newcommand{\fig}[1]{Fig.~\ref{fig:#1}}

\newcommand{\fign}[1]{\ref{fig:#1}}
\newcommand{\tbl}[1]{Table~\ref{tab:#1}}
\newcommand{\sctn}[1]{Sec.~\ref{sec:#1}}

\newcommand{\mb}[1]{\mathbf{#1}}

\providecommand{\norm}[1]{\left\lVert#1\right\rVert}
\DeclarePairedDelimiterX{\normsz}[1]{\lVert}{\rVert}{#1}

\DeclareMathOperator*{\argmin}{arg\,min}

\def \co {\mathcal{O}}

\def \E  {\mathbb{E}}
\def \R  {\mathbb{R}}
\def \C  {\mathbb{C}}
\def \F  {\mathcal{F}}
\def \G  {\mathcal{G}}

\def \T {\mathbf{T}}
\def \fpr {\mathbf{R}}
\newtheorem{assume}{Assumption}
\newtheorem{prop}{Proposition}

\date{\today}

\begin{document}

\title{First and Second Order Methods for Online Convolutional Dictionary Learning\thanks{Published in SIAM J. Imaging Sci., 11(2), 1589--1628, 2018. doi:10.1137/17M1145689
\funding{This research was supported by the U.S. Department of Energy through the LANL/LDRD Program. The work of Jialin Liu and Wotao Yin was supported in part by NSF grant DMS-1720237 and ONR grant N000141712162.}}}

\author{Jialin Liu\thanks{Department of Mathematics, UCLA, Los
    Angeles, CA 90095 (\texttt{danny19921123@gmail.com})}
  \and Cristina Garcia-Cardona\thanks{CCS Division, Los Alamos
    National Laboratory, Los Alamos, NM 87545 (\texttt{cgarciac@lanl.gov})}
  \and Brendt Wohlberg\thanks{T Division, Los Alamos
    National Laboratory, Los Alamos, NM 87545  (\texttt{brendt@lanl.gov})}
   \and Wotao Yin\thanks{Department of Mathematics, UCLA, Los
    Angeles, CA 90095 (\texttt{wotaoyin@math.ucla.edu})}
}

\maketitle

\begin{abstract}
Convolutional sparse representations are a form of sparse representation with a structured, translation invariant dictionary.  Most convolutional dictionary learning algorithms to date operate in batch mode, requiring simultaneous access to all training images during the learning process, which results in very high memory usage, and severely limits the training data size that can be used.  Very recently, however, a number of authors have considered the design of online convolutional dictionary learning algorithms that offer far better scaling of memory and computational cost with training set size than batch methods.  This paper extends our prior work, improving a number of aspects of our previous algorithm; proposing an entirely new one, with better performance, and that supports the inclusion of a spatial mask for learning from incomplete data; and providing a rigorous theoretical analysis of these methods.
\end{abstract}

\begin{keywords}
convolutional sparse coding, convolutional dictionary learning, online dictionary learning, stochastic gradient descent, recursive least squares
\end{keywords}

\section{Introduction}
\label{sec:intro}

\subsection{Sparse representations and dictionary learning}

\emph{Sparse signal representation} aims to represent a given signal by a linear combination of only a few elements of a fixed set of signal components~\cite{mairal-2014-sparse}. For example, we can approximate an $N$-dimensional signal $\mb{s}\in\R^N$ as
\begin{equation}
\label{eqn:sparse-coding}
\mb{s} \approx D \mb{x} = \mb{d}_1 x_1 + \ldots + \mb{d}_M x_M \;,
\end{equation}
where $D = [\mb{d}_1, \mb{d}_2, \cdots, \mb{d}_M] \in \R^{N\times M}$
is the \emph{dictionary} with $M$ \emph{atoms} and $\mb{x} = [x_1, x_2, \cdots, x_M]^T \in\R^M$ is the sparse representation. The problem of computing the sparse representation $\mb{x}$ given $\mb{s}$ and $D$ is referred to as \emph{sparse coding}. Among a variety of formulations of this problem, we focus on Basis Pursuit Denoising (BPDN)~\cite{chen-1998-atomic}
\begin{equation}
\label{eqn:bpdn}
\min_{\mb{x}} \; (1/2)\norm{D\mb{x}-\mb{s}}_2^2 + \lambda \norm{\mb{x}}_1 \;.
\end{equation}

Sparse representations have been used in a wide variety of applications, including denoising~\cite{elad2006image, mairal-2014-sparse}, super-resolution \cite{yang2010image, zhang2015survey}, classification \cite{wright2009robust}, and face recognition \cite{wright2010sparse}. A key issue when solving sparse coding problems as in (\ref{eqn:bpdn}) is how to choose the dictionary $D$. Early work on sparse representations used a fixed basis \cite{rubinstein2010dictionaries} such as wavelets \cite{mallat1999wavelet} or Discrete Cosine Transform (DCT) \cite{huang2007sparse},  but learned dictionaries can provide better performance \cite{aharon2006rm, elad2006image}.

\emph{Dictionary learning} aims to learn a good dictionary $D$ for a given distribution of signals. If $\mb{s}$ is a random variable,
the dictionary learning problem can be formulated as
\begin{equation}
\label{eqn:dl}
\min_{D\in C} \E_{\mb{s}} \bigg\{ \min_{\mb{x}} \frac{1}{2}\norm{D\mb{x}-\mb{s}}_2^2 + \lambda \norm{\mb{x}}_1\bigg\} \;,
\end{equation}
where $C = \{D \;|\; \norm{\mb{d}_m}_2^2 \leq1,\forall m\}$ is the constraint set, which is necessary to resolve the scaling ambiguity between $D$ and $\mb{x}$.

\emph{Batch dictionary learning methods} (e.g.~\cite{engan1999frame, engan1999method, aharon2006rm, xu-2016-fast}) sample a batch of training signals $\{\mb{s}_1, \mb{s}_2, \ldots, \mb{s}_K\}$ before training, and minimize an objective function such as
\begin{equation}
\label{eqn:batch-dl}
\min_{D\in C, \mb{x}} \sum_{k=1}^K \bigg\{ \frac{1}{2}\norm{D\mb{x}_k -\mb{s}_k}_2^2 + \lambda \norm{\mb{x}_k}_1\bigg\} \;.
\end{equation}
These methods require simultaneous access to all the training samples during training.

In contrast, \emph{online dictionary learning methods} process training samples in a streaming fashion. Specifically, let $\mb{s}^{(t)}$ be the chosen sample at the $t^{\text{th}}$ training step. The framework of online dictionary learning is
\begin{equation}
\label{eqn:online-dl}
\begin{split}
\mb{x}^{(t)} &= \text{SC} \Big(D^{(t-1)}; \mb{s}^{(t)}\Big) \;,\\
D^{(t)} &= D\text{-update} \bigg(\{D^{(\tau)}\}_{\tau=0}^{t-1}, \{\mb{x}^{(\tau)}\}_{\tau=1}^{t}, \{\mb{s}^{(\tau)}\}_{\tau=1}^{t}\bigg) \;.
\end{split}
\end{equation}
where SC denotes sparse coding, for instance, (\ref{eqn:bpdn}), and $D$-update computes a new dictionary $D^{(t)}$ given the past information $\{D^{(\tau)}\}_{\tau=0}^{t-1}, \{\mb{x}^{(\tau)}\}_{\tau=1}^{t}, \{\mb{s}^{(\tau)}\}_{\tau=1}^{t}$.  While each outer iteration of a batch dictionary learning algorithm involves computing the coefficient maps $\mb{x}_k$ for all training samples, online learning methods compute the coefficient map $\mb{x}^{(t)}$ for only one, or a small number, of training sample $\mb{s}^{(t)}$ at each iteration, the other coefficient maps $\{\mb{x}^{(\tau)}\}_{\tau=1}^{t-1}$ used in the $D$-update having been computed in previous iterations. Thus, these algorithms can be implemented for large sets of training data or dynamically generated data. Online $D$-update methods and the corresponding online dictionary learning algorithms can be divided into two classes:

\emph{Class I: first-order algorithms} \cite{wang2010locality, mairal2012task, aharon2008sparse} are inspired by Stochastic Gradient Descent (SGD), which only uses first-order information, the gradient of the loss function, to update the dictionary $D$.

\paragraph{Class II: second-order algorithms}
These algorithms are inspired by Recursive Least Squares (RLS) \cite{skretting-2010-recursive, eksioglu2014online}, Iterative Reweighted Least Squares (IRLS) \cite{lu2013online, wang2013online}, Kernel RLS \cite{gao2014online}, second-order Stochastic Approximation (SA) \cite{mairal2009online, szabo2011online, zhang2012online, slavakis2014online, zhang2015online, kasiviswanathan2012online}, etc. They use previous information $\{D^{(\tau)}\}_{\tau=0}^{t-1}, \{\mb{x}^{(\tau)}\}_{\tau=1}^{t}, \{\mb{s}^{(\tau)}\}_{\tau=1}^{t}$ to construct a surrogate function $\F^{(t)}(D)$ to estimate the true loss function of $D$ and then update $D$ by minimizing this surrogate function. These surrogate functions involve both first-order and second-order information, i.e. the gradient and Hessian of the loss function, respectively.

The most significant difference between the two classes is that Class I algorithms only need access to information from the current step, $t$, \ie $D^{(t-1)}, \mb{x}^{(t)},\mb{s}^{(t)}$, while Class II algorithms use the entire history up to step $t$, \ie $\{D^{(\tau)}\}_{\tau=0}^{t-1}$, $\{\mb{x}^{(\tau)}\}_{\tau=1}^{t}$, $\{\mb{s}^{(\tau)}\}_{\tau=1}^{t}$. However, as we discuss in~\sctn{surrogate} below, it is possible to store this information in aggregate form so that the memory requirements do not scale with $t$.

\subsection{Convolutional form}

\emph{Convolutional Sparse Coding (CSC)} \cite{lewicki-1999-coding, zeiler-2010-deconvolutional} \cite[Sec. II]{wohlberg-2016-efficient}, a highly structured sparse representation model, has recently attracted increasing attention for a variety of imaging inverse problems~\cite{gu-2015-convolutional, liu-2016-image, zhang-2016-convolutional, quan-2016-compressed, wohlberg-2016-convolutional2, zhang-2017-convolutional}. CSC aims to represent a given signal $\mb{s}\in\R^N$ as a sum of convolutions,
\begin{equation}
\label{eqn:csc}
\mb{s} \approx \mb{d}_1 \ast \mb{x}_1 + \ldots + \mb{d}_M \ast \mb{x}_M \;,
\end{equation}
where dictionary atoms $\{\mb{d}_m\}_{m=1}^M$ are linear filters and the representation $\{\mb{x}_m\}_{m=1}^M$ is a set of \emph{coefficient maps}, each map $\mb{x}_m$ having the same size $N$\ as the signal $\mb{s}$.  Since we implement the convolutions in the frequency domain for computational efficiency, it is convenient to adopt \emph{circular boundary conditions} for the convolution operation.

Given $\{\mb{d}_m\}$ and $\mb{s}$, the maps $\{\mb{x}_m\}$ can be obtained by solving the Convolutional Basis Pursuit DeNoising (CBPDN) $\ell_1$-minimization problem
\begin{align}
  \min_{\{\mb{x}_{m}\}} \frac{1}{2}
  \normsz[\Big]{\sum_{m=1}^M \mb{d}_m \ast \mb{x}_{m} - \mb{s}}_2^2 +
  \lambda  \sum_{m=1}^M \norm{\mb{x}_{m}}_1  \; .
\label{eq:cbpdn}
\end{align}
The corresponding dictionary learning problem is called \emph{Convolutional Dictionary Learning (CDL)}.  Specifically, given a set of $K$ training signals $\{\mb{s}_k\}_{k=1}^K$, CDL is implemented via minimization of the function
\begin{align}
\label{eqn:batch-cdl}
  \min_{\{\mb{d}_{m}\},\{\mb{x}_{k,m}\}} \frac{1}{2} \sum_{k=1}^K
  \normsz[\Big]{\sum_{m=1}^M \mb{d}_m \ast \mb{x}_{k,m} - \mb{s}_k}_2^2 +
  \lambda \sum_{k=1}^K \sum_{m=1}^M \norm{\mb{x}_{k,m}}_1   \nonumber \\
 \text{ subject to}  \norm{\mb{d}_{m}}_2 \leq 1,
\; \forall m \in \{1,\ldots,M\} \; ,
\end{align}
where the coefficient maps ${\mb{x}_{k,m}}$, $k \in \{1,\ldots,K\}$, $m \in \{1,\ldots,M\}$, represent $\mb{s}_k$, and the norm constraint avoids the scaling ambiguity between $\mb{d}_m$ and $\mb{x}_{k,m}$.

The \emph{masked CDL problem}~\cite{heide-2015-fast, wohlberg-2016-boundary}, which is able to learn the convolutional dictionary from training signals with missing samples, is a variant of (\ref{eqn:batch-cdl}),
\begin{align}
\label{eqn:batch-cdl-ms}
  \min_{\{\mb{d}_{m}\},\{\mb{x}_{k,m}\}} \frac{1}{2} \sum_{k=1}^K
  \normsz[\Big]{\sum_{m=1}^M W \odot \big( \mb{d}_m \ast \mb{x}_{k,m} - \mb{s}_k\big)}_2^2 +
  \lambda \sum_{k=1}^K \sum_{m=1}^M \norm{\mb{x}_{k,m}}_1   \nonumber \\
 \text{ subject to}  \norm{\mb{d}_{m}}_2 \leq 1,
\; \forall m \in \{1,\ldots,M\} \; ,
\end{align}
where the \emph{masking matrix} $W$ is usually a $\{0,1\}$-valued matrix that masks unknown or unreliable pixels, and operator $\odot$ denotes pointwise multiplication.

\paragraph{Complexity of batch CDL}
Most current CDL algorithms \cite{bristow-2013-fast, heide-2015-fast, gu-2015-convolutional, wohlberg-2016-efficient, sorel-2016-fast, wohlberg-2016-boundary, garcia-2017-subproblem, chun-2017-convolutional} are batch learning methods that alternatively minimize over $\{\mb{x}_{k,m}\}$ and $\{\mb{d}_{m}\}$, dealing with the entire training set at each iteration.
When $K$ is large, the $\mb{d}_m$ update subproblem is computationally expensive, e.g. the single step complexity and memory usage are both $\co(KMN\log(N))$ for one of the current state-of-the-art methods~\cite{sorel-2016-fast, garcia-2017-subproblem}. For example, for a medium-sized problem with $K=40, N=256\times256, M=64$, we have $KMN\log(N) \approx 10^9$, which is computationally very expensive.

\subsection{Contribution of this article}

The goal of the present work is to develop online convolutional dictionary learning methods for training data sets that are much larger than those that are presently feasible.  We develop online methods for CDL in two directions: first-order method and second-order method. The contribution of this article includes:
\begin{enumerate}
\item An efficient first-order online CDL method (Algorithm \ref{algo:sgd}), that provides a new framework with lower learning time and memory requirements than our previous state-of-the-art online method~\cite{liu-2017-online}.
\item An efficient second-order online CDL method (Algorithm
  \ref{algo:surro-splitting}) that improves the algorithm proposed
  in~\cite{liu-2017-online} and proves its convergence.
\item An online CDL method for masked images, which is, to the best of our knowledge, the first online algorithm able to learn dictionaries from partially masked training set.
\item An analysis of the forgetting factor\footnote{This technique is used in previous works \cite{skretting-2010-recursive, mairal-2010-online, szabo2011online, slavakis2014online}, but not theoretically analyzed.} used in Algorithm \ref{algo:surro-splitting}.
\item An analysis of the stopping condition in the $D$-update.
\item An analysis of the effects of circular boundary conditions on dictionary learning.
\end{enumerate}

\paragraph{Relationship with other works}
Recently, two other works on online CDL \cite{degraux-2017-online, wang2017online} have appeared. Both of them study second-order SA methods. They use the same framework as \cite{mairal2009online} but different methods to update $D$: \cite{degraux-2017-online} uses projected coordinate descent and \cite{wang2017online} uses the iterated Sherman-Morrison update~\cite{wohlberg-2016-efficient}. Our previous work \cite{liu-2017-online} uses frequency-domain FISTA to update $D$, with a \emph{forgetting factor} technique inspired by \cite{skretting-2010-recursive, mairal-2010-online} to correct the surrogate function, and uses ``region-sampling'' to reduce the memory cost.  In this paper, the second-order SA algorithm, Algorithm \ref{algo:surro-splitting}, improves the algorithm in~\cite{liu-2017-online} by introducing two additional techniques: an improved stopping condition for FISTA and image-splitting. The former technique greatly reduces the number of inner-loop iterations of the $D$-update. The latter, compared with ``region-sampling'', fully utilizes all the information in the training set. With these techniques, Algorithm \ref{algo:surro-splitting} converges faster than the algorithm in~\cite{liu-2017-online}.

The method in \cite{degraux-2017-online} is designed for a different problem than (\ref{eqn:batch-cdl}), and is therefore not directly comparable with our methods. The other recent paper on online CDL~\cite{wang2017online}, which appeared while we were completing this work, proposed an algorithm that uses the same framework as our previous work \cite{liu-2017-online}, and is therefore expected to offer similar performance to our initial method.

\section{Preliminaries}
\label{sec:pre}

Here we introduce our notation. The signal is denoted by $\mb{s}\in\R^N$, and the dictionaries by $\mb{d} = (\mb{d}_1 \; \mb{d}_2 \; \ldots\; \mb{d}_M)^T \in \R^{ML}$, where the dictionary kernels (or filters) are $\mb{d}_m \in \R^{L}$.  The coefficient maps are denoted by $\mb{x} = (\mb{x}_1 \; \mb{x}_2 \; \ldots\; \mb{x}_M)^T \in \R^{MN}$, where $\mb{x}_m\in\R^{N}$ is the coefficient map corresponding to $\mb{d}_m$.  In addition to the \emph{vector form}, $\mb{x}$, of the coefficient maps, we define an \emph{operator form} $X$. First we define a linear operator $X_m$ on $\mb{d}_m$ such that $X_{m}\mb{d}_m=\mb{d}_m \ast \mb{x}_{m}$ and let $X \triangleq \big(X_{1}~X_{2}\cdots X_{M}\big)$. Then, we have
\begin{equation}\label{eqn:operatorX}
 X \mb{d} \triangleq \sum_{m=1}^MX_m\mb{d}_m = \sum_{m=1}^M \mb{d}_m \ast \mb{x}_{m} \approx \mb{s} \;.
\end{equation}
Hence, $X : \R^{ML}\to\R^{N}$, a linear operator defined from the dictionary space to the signal space, is the operator form of $\mb{x}$.

\subsection{Problem settings}
\label{sec:settings}

Now we reformulate (\ref{eqn:batch-cdl}) into a more general form. (The masked problem (\ref{eqn:batch-cdl-ms}) will be discussed in Section \ref{sec:msk}.) Usually, the signal is sampled from a large training set, but we consider the training signal $\mb{s}$ as a random variable following the distribution $\mb{s} \sim P_S (\mb{s})$.  Our goal is to optimize the dictionary $\mb{d}$.  Given $\mb{s}$, the loss function $l$ to evaluate $\mb{d},\mb{x}$ is defined as
\begin{equation}
\label{eqn:basic_l}
l(\mb{d},\mb{x}; \mb{s}) = (1/2) \norm{X \mb{d} - \mb{s}}_2^2 + \lambda \norm{\mb{x}}_1 \;.
\end{equation}
Given $\mb{s}$, the loss function $f$ to evaluate $\mb{d}$ and the corresponding minimizer are respectively,
\begin{equation}
f(\mb{d}; \mb{s}) \triangleq \min_{\mb{x}} l(\mb{d},\mb{x};\mb{s})
 \quad \text{and} \quad
\mb{x}^*(\mb{d};\mb{s}) \triangleq \argmin_{\mb{x}}  l(\mb{d},\mb{x};\mb{s}) \;. \label{eqn:cbpdn}
\end{equation}
A general CDL problem can be formulated as
\begin{equation}
\label{eqn:cdl}
\min_{\mb{d}\in \text{C}} \E_{\mb{s}}[ f(\mb{d}; \mb{s})]\;,
\end{equation}
where $\text{C}$ is the constraint set of
$\text{C}=\{ \mb{d} \;| \; \norm{\mb{d}_m}^2\leq1, \forall m \}$.

\subsection{Two online frameworks}
\label{sec:framework}

Now we consider the CDL problem (\ref{eqn:cdl}) when the training signals $\mb{s}^{(1)}, \mb{s}^{(2)}, \cdots, \mb{s}^{(t)},\cdots$ arrive in a streaming fashion. Inspired by online methods for standard dictionary learning problems,
we propose two online frameworks for CDL problem (\ref{eqn:cdl}). One is a first order method based on Projected Stochastic Gradient Descent (SGD)~\cite{wang2010locality, mairal2012task, aharon2008sparse}:
\begin{equation}
\label{eqn:psgd}
\mb{d}^{(t)} = \text{Proj}_{\text{C}}\Big(\mb{d}^{(t-1)} - \eta^{(t)}\nabla f\big(\mb{d}^{(t-1)}; \mb{s}^{(t)}\big) \Big) \;.
\end{equation}

The other is a second order method, which is inspired by least squares estimator for dictionary learning~\cite{mairal2009online, skretting-2010-recursive, szabo2011online, zhang2012online, slavakis2014online, zhang2015online, kasiviswanathan2012online}.  A naive least squares estimator can be written as
\[
\mb{d}^{(t)} = \argmin_{\mb{d}\in\text{C}}\Big\{ \underbrace{\min_{\mb{x}}\ell(\mb{d},\mb{x},\mb{s}^{(1)}) + \cdots + \min_{\mb{x}}\ell(\mb{d},\mb{x},\mb{s}^{(t)})}_{\text{Objective function on training samples $F^{(t)}(\mb{d})$}} \Big\} \;.
\]
This is not practical because the inner minimizer of $\mb{x}$ depends on $\mb{d}$, which is unknown.
To solve this problem, we can fix $\mb{d}$ when we minimize over $\mb{x}$, i.e.
\begin{subequations}
\begin{align}
\mb{x}^{(t)} =& \argmin_{\mb{x}} \ell(\mb{d}^{(t-1)},\mb{x};\mb{s}^{(t)}).\label{eqn:cbpdn_t}\\
\mb{d}^{(t)} =& \argmin_{\mb{d}\in\text{C}} \Big\{ \underbrace{\ell(\mb{d},\mb{x}^{(1)},\mb{s}^{(1)}) + \cdots + \ell(\mb{d},\mb{x}^{(t)},\mb{s}^{(t)})}_{\text{Surrogate function } \F^{(t)}(\mb{d})} \Big\} \;.
\end{align}
\label{eqn:surro}
\end{subequations}
Direct application of these methods to the CDL problem is very computationally expensive, but we propose a number of techniques to reduce the time and memory usage.  The details are discussed in Sections \ref{sec:sgd} and \ref{sec:surrogate} respectively.

\subsection{Techniques to calculate operator $X$}
\label{sec:freq}

Before introducing our algorithms for (\ref{eqn:cdl}), we consider a basic problem and two computational techniques that are used in this section as well as in Sections \ref{sec:sgd} and \ref{sec:surrogate}.

With $\mb{s}$ and $\mb{x}$ fixed, the basic problem is
\begin{equation}
\label{eqn:basic}
\min_{\mb{d}\in\R^{ML}} l(\mb{d},\mb{x};\mb{s}) + \iota_{\text{C}}(\mb{d}) \;,
\end{equation}
where $\iota_{\mathrm{C}}(\cdot)$ is the indicator function\footnote{The indicator function is defined as: $\iota_{\mathrm{C}}(\mb{d})=\begin{dcases} 0,& \text{if } \mb{d} \in \text{C}\\     +\infty,              & \text{otherwise}\end{dcases}$.} of set $\mathrm{C}$.
To solve this problem we can apply projected gradient descent (GD)
\cite{bertsekas1999nonlinear}
\begin{equation}
\label{eqn:basic_gd}
\mb{d}^{(t)} = \text{Proj}_{\text{C}}\Big(\mb{d}^{(t-1)} - \eta^{(t)}X^T \big(X \mb{d}^{(t-1)} - \mb{s}\big) \Big)\;,
\end{equation}
where $(t)$ is the iteration index and $ X^T \big(X \mb{d} - \mb{s}\big)$ is the gradient of $l$ with respect to $\mb{d}$.
Since $X$ is a linear operator from $\R^{ML}$ to $\R^{N}$, the cost of directly computing (\ref{eqn:basic_gd}) is $\co(NM
L)$.
However, we can exploit the sparsity or the structure of operator $X$ to yield a more efficient computation that greatly reduces the time complexity.

\subsubsection{Computing with sparsity property}

The first option is to utilize the sparsity of $X$. Specifically, $X$ is saved as a triple array $(i,j,v)$, which records the indices $(i,j)$ and values $v$ of the non-zero elements of $X$, so that only the nonzero entries in $X$ contribute to the computational time. This triple array is commonly referred as a \emph{coordinate list} and is a standard way of representing a sparse matrix.

Let us compute the non-zero entries of operator $X$. The operator form $X_m$ of the $N$-dimensional vectors $\mb{x}_m = ((x_m)_1, \cdots, (x_m)_N)^T$ can be written as
\[X_m =
\begin{bmatrix*}[l]
    (x_m)_1       & (x_m)_N & (x_m)_{N-1} & \dots & (x_m)_{N-L+2} \\
    (x_m)_2      & (x_m)_1 & (x_m)_N & \dots & (x_m)_{N-L+3}  \\
(x_m)_3 &               (x_m)_2 & (x_m)_1 & \dots & (x_m)_{N-L+4}  \\
  \vdots & \vdots & \vdots & \ddots & \vdots \\
    (x_m)_N      & (x_m)_{N-1} & (x_m)_{N-2} & \dots & (x_m)_{N-L+1}
\end{bmatrix*},\]
where each column is a circular shift of $\mb{x}_m$ and $L$ is the dimension of each dictionary kernel. Thus, the density of $\mb{x}_m$ and $X_m$ are the same. Assuming the density of vector $\mb{x}$ is $\rho$, the number of nonzero entries of operator $X$ is $NML\rho$, giving a single step complexity of $\co(NML\rho)$ for computing (\ref{eqn:basic_gd}).

\subsubsection{Computing in the frequency domain} Another option is to utilize the structure of $X$.
It is well known that convolving two signals of the same size corresponds to the pointwise multiplication of their frequency representations. Our method below takes advantage of this property.  First, we zero-pad each $\mb{d}_m$ from $\R^L$ to $\R^N$ to match the size of $\mb{s}$. Then the basic problem can be written as
\vspace{-1mm}
\begin{equation}
\label{eqn:basic_zeropad}\min_{\mb{d}\in\R^{MN}} l(\mb{d},\mb{x};\mb{s}) + \iota_{\text{C}_\text{PN}}(\mb{d}) \;,
\vspace{-1mm}
\end{equation}
where the set $\text{C}_\text{PN}$ is defined as
\vspace{-1mm}
\begin{equation}
\label{eqn:cpn}
\text{C}_\text{PN} \triangleq \{ \mb{d}_m \in \R^N: (I-P)\mb{d}_m=0, \norm{\mb{d}_m}^2\leq1 \} \;.
\vspace{-1mm}
\end{equation}
Operator $P$ preserves the desired support of $\mb{d}_m$ and masks the remaining part to zeros.  Projected GD (\ref{eqn:basic_gd}) has an equivalent form:
\vspace{-1mm}
\begin{equation}
\label{eqn:basic_gd_spatial}
\mb{d}^{(t)} = \text{Proj}_{\text{C}_\text{PN}}\Big(\mb{d}^{(t-1)} - \eta^{(t)}\frac{\partial l}{\partial \mb{d}}(\mb{d}^{(t-1)},\mb{x}; \mb{s}) \Big) \;.
\vspace{-1mm}
\end{equation}

Then, using the Plancherel formula, we can write the loss function\footnote{We ignore the term $\lambda\norm{\mb{x}}_1$ in $l$ here because $\mb{x}$ is fixed in this problem.} $l$ as
\vspace{-1mm}
\begin{equation}
\label{eqn:basic_l_freq}
l(\mb{d},\mb{x};\mb{s}) = \underbrace{ \normsz[\Big]{\sum_m \mb{d}_m \ast \mb{x}_{m} - \mb{s}}_2^2}_{\norm{X\mb{d} - \mb{s}}^2} = \underbrace{\normsz[\Big]{\sum_m \hat{\mb{d}}_m \odot \hat{\mb{x}}_{m} - \hat{\mb{s}}}_2^2}_{\norm{\hat{X}\hat{\mb{d}} - \hat{\mb{s}}}^2} \;,
\vspace{-1mm}
\end{equation}
where $\hat{\cdot}$ denotes the corresponding quantity in the frequency domain and $\odot$ means pointwise multiplication. Therefore, we have $\hat{\mb{d}} \in \C^{MN}$, and $\hat{X}=\big(\hat{X}_{1}~\hat{X}_{2}\cdots \hat{X}_{M}\big)$ is a linear operator. Define the loss function in the frequency domain
\vspace{-1mm}
\begin{equation}
\label{eqn:basic_l_freq_def}
\hat{l}(\hat{\mb{d}},\hat{\mb{x}};\hat{\mb{s}})  = (1/2) \normsz[\big]{\hat{X}\hat{\mb{d}} - \hat{\mb{s}}}^2 \;,
\vspace{-1mm}
\end{equation}
which is a real valued function defined in the complex domain.
The Cauchy-Riemann condition \cite{ahlfors1953complex} implies that (\ref{eqn:basic_l_freq_def}) is not differentiable unless it is constant. However, the \emph{conjugate cogradient}\footnote{The conjugate cogradient of function $f(x):\C^n\to\R$ is defined as: $\frac{\partial f}{\partial \Re(x)} + i \frac{\partial f}{\partial \Im(x)}$, where $\Re(x)$, $I\Im(x)$ are the real part and imaginary part of $x$. The derivation of (\ref{eqn:basic_l_freq_def}) is given in Appendix~\ref{app:derivation}.}~\cite{sorber2012unconstrained}
\begin{equation}
\label{eqn:basic_gradient_freq}
\frac{\partial \hat{l}}{\partial \hat{\mb{d}}}(\hat{\mb{d}},\hat{\mb{x}};\hat{\mb{s}}) \triangleq \hat{X}^H(\hat{X} \hat{\mb{d}} - \hat{\mb{s}}) \;.
\end{equation}
exists and can be used for minimizing (\ref{eqn:basic_l_freq_def}) by gradient descent.

Since each item $\hat{X}_m$ in $\hat{X}$ is diagonal, the gradient is easy to compute, with a complexity of $\co(NM)$, instead of $\co(NML)$. Based on (\ref{eqn:basic_gradient_freq}), we have the following modified gradient descent:
\begin{equation}
\label{eqn:basic_gd_freq}
\mb{d}^{(t)} = \text{Proj}_{\text{C}_\text{PN}}\bigg(\text{IFFT}\Big(\hat{\mb{d}}^{(t-1)} - \eta^{(t)}\frac{\partial \hat{l}}{\partial \hat{\mb{d}}}\big(\hat{\mb{d}}^{(t-1)},\hat{\mb{x}}; \hat{\mb{s}}\big) \Big)\bigg) \;.
\end{equation}
To compute (\ref{eqn:basic_gd_freq}), we transform $\mb{d}^{(t)}$ into its frequency domain counterpart $\hat{\mb{d}}^{(t)}$, perform gradient descent in the frequency domain, return to the spatial domain, and project the result onto the set $\text{C}_\text{PN}$.

In our modified method (\ref{eqn:basic_gd_freq}), the iterate $\mb{d}^{(t)}$ is transformed between the frequency and spatial domains because the gradient is cheaper to compute in the frequency domain, but projection is cheaper to compute in the spatial domain.

\textbf{Equivalence of (\ref{eqn:basic_gd_spatial}) and (\ref{eqn:basic_gd_freq}).} We can prove
\begin{equation}
\label{eqn:equal_freq_spatial}
\hat{X}^H(\hat{X} \hat{\mb{d}} - \hat{\mb{s}}) = \text{FFT}\big(X^T(X\mb{d}-\mb{s})\big)\; , \quad \forall \mb{x},\mb{d},\mb{s} \;,
\end{equation}
which means that the conjugate cogradient of $\hat{l}$ is equivalent to the gradient of $l$. Thus, modified GD (\ref{eqn:basic_gd_freq}) coincides with standard GD (\ref{eqn:basic_gd_spatial}) using conjugate cogradient.

A proof of (\ref{eqn:equal_freq_spatial}) given in Appendix \ref{app:equal}. A similar result is also given in \cite{rippel2015spectral} under the name ``conjugate symmetry''.

\section{First-order method: Algorithm \ref{algo:sgd}}
\label{sec:sgd}

Recall the Projected SGD step (\ref{eqn:psgd})
\[
\mb{d}^{(t)} = \text{Proj}_{\text{C}}\Big(\mb{d}^{(t-1)} - \eta^{(t)}\nabla f(\mb{d}^{(t-1)}; \mb{s}^{(t)}) \Big) \;,
\]
where parameter $\eta^{(t)}$ is the \emph{step size}\footnote{Some authors refer to it as the \emph{learning rate}.}. Given the definition of $f$ in (\ref{eqn:cbpdn}), $\nabla f(\mb{d}^{(t-1)}; \mb{s}^{(t)})$ is the partial derivative with respect to $\mb{d}$ at the optimal $\mb{x}$ \cite{mairal-2010-online,danskin1966theory}, i.e. $\nabla f(\mb{d};\mb{s}) = \frac{\partial l}{\partial \mb{d}} (\mb{d},\mb{x}^*(\mb{d},\mb{s}); \mb{s})$, where $\mb{x}^*$ is defined by (\ref{eqn:cbpdn}).

Thus, to compute the gradient $\nabla f(\mb{d}^{(t-1)}; \mb{s}^{(t)})$, we should first compute the coefficient maps $\mb{x}^{(t)}$ of the $t^{\mathrm{th}}$ training signal $\mb{s}^{(t)}$ with dictionary $\mb{d}^{(t-1)}$, which is given by (\ref{eqn:cbpdn_t}).  Then we can compute the gradient as
\[\nabla f\big(\mb{d}^{(t-1)}; \mb{s}^{(t)}\big) =\frac{\partial l}{\partial \mb{d}} \big(\mb{d}^{(t-1)},\mb{x}^{(t)}; \mb{s}^{(t)} \big) = \big(X^{(t)}\big)^T\Big(X^{(t)}\mb{d}^{(t-1)}-\mb{s}^{(t)}\Big) \;.
\]

Based on the discussion in Section \ref{sec:freq}, we can perform gradient descent either in the spatial-domain or the frequency-domain. In the frequency domain, the conjugate cogradient of $\nabla \hat{f}$ is:
\[
\nabla \hat{f}(\hat{\mb{d}}^{(t-1)}; \hat{\mb{s}}^{(t)}) = \frac{\partial \hat{l}}{\partial \hat{\mb{d}}} \big(\hat{\mb{d}}^{(t-1)},\hat{\mb{x}}^{(t)}; \hat{\mb{s}}^{(t)} \big) = \big(\hat{X}^{(t)}\big)^H \Big(\hat{X}^{(t)}\hat{\mb{d}}^{(t-1)}-\hat{\mb{s}}^{(t)}\Big) \;.
\]
The full algorithm is summarized in Algorithm \ref{algo:sgd}.

\begin{algorithm2e}[h]
\SetKwInOut{initial}{Initialize}
\initial{Initialize $\mb{d}^{(0)}$ with a random dictionary.}
\For{$t=1,\cdots, T$} {
Sample a signal $\mb{s}^{(t)}$.\\
Solve convolutional sparse coding problem (\ref{eqn:cbpdn_t}) to obtain $\mb{x}^{(t)}$.\\
\uIf{Option I} {Update dictionary in the spatial-domain with sparse matrix $X^{(t)}$:
\vspace{-2mm}
\[
\mb{d}^{(t)} = \text{Proj}_{\text{C}}\Big(\mb{d}^{(t-1)} - \eta^{(t)}\big(X^{(t)}\big)^T\big(X^{(t)}\mb{d}^{(t-1)}-\mb{s}^{(t)}\big)\Big)
\vspace{-2mm}
\]
}
\ElseIf{Option II }{Update dictionary in the frequency-domain:
\vspace{-2mm}
\[
\begin{aligned}
\hat{\mb{x}}^{(t)} = \;& \text{FFT}(\mb{x}^{(t)})\\
\mb{d}^{(t)} = \;& \text{Proj}_{\text{C}_\text{PN}}\bigg(\text{IFFT}\Big(\hat{\mb{d}}^{(t-1)} - \eta^{(t)}\big(\hat{X}^{(t)}\big)^H\big(\hat{X}^{(t)}\hat{\mb{d}}^{(t-1)}-\hat{\mb{s}}^{(t)}\big)\Big)\bigg)
\end{aligned}
\]
\vspace{-5mm}}

}
\KwOut{$\mb{d}^{(T)}$}
\caption{Online Convolutional Dictionary Learning (Modified SGD)}\label{algo:sgd}
\end{algorithm2e}

\textbf{Complexity analysis of Algorithm \ref{algo:sgd}.} We list the single-step complexity and memory usage of different options in Table \ref{tab:summary}. Both the frequency-domain update and sparse matrix technique reduce single-step complexities. The comparison between these two computational techniques depends on the sparsity of $X^{(t)}$ and the dictionary kernel size $L$.  In Section \ref{sec:sgd_rslt}, we will numerically compare these methods.

\begin{table}[t]
\centering
\begin{tabular}{|b{0.28\linewidth}|b{0.45\linewidth}|b{0.18\linewidth}|}
\hline
Scheme & Single step complexity & Memory usage\\
\hline \hline
Spatial (dense matrix) & $T_{\text{CBPDN}} + \co(NML)$ & $\co(NML)$\\
\hline
Spatial (sparse matrix) & $T_{\text{CBPDN}} + \co(NML\rho)$ & $\co(NML\rho)$\\
\hline
Frequency update & $T_{\text{CBPDN}} + \co(NM\log(N)) + \co(NM)$ & $\co(MN)$\\
\hline
\end{tabular}
\caption{Single step complexity and memory usage of Algorithm \ref{algo:sgd}. $N$: signal dimension; $M$: number of dictionary kernels; $L$: size of each kernel; $\rho$: average density of the coefficient maps.
}
\label{tab:summary}
\end{table}

\textbf{Convergence of Algorithm \ref{algo:sgd}.} Algorithm \ref{algo:sgd}, by (\ref{eqn:equal_freq_spatial}), is equivalent to the standard projected SGD. Thus, by properly choosing step sizes $\eta^{(t)}$, Algorithm \ref{algo:sgd} converges to a stationary point \cite{ghadimi2013stochastic}. A diminishing step size rule $\eta^{(t)} = a / (b+t)$ is used in other dictionary learning works \cite{aharon2008sparse,mairal2009online}. The convergence performance with different step sizes are numerically tested in Section \ref{sec:sgd_rslt}.

\section{Second-order method: Algorithm \ref{algo:surro-splitting}}
\label{sec:surrogate}

In this section, we first introduce some details of directly applying second order stochastic approximation method (\ref{eqn:surro}) to CDL problems, then we discuss some issues and our resolutions.

Aggregating the true loss function $f(\mb{d};\mb{s}^{(t)})$ on the $t^{\text{th}}$ sample $\mb{s}^{(t)}$, the objective function on the first $t$ training samples is
\begin{equation}
\label{eqn:F_standard}
F^{(t)}(\mb{d}) = \frac{1}{t}\Big(\sum_{\tau=1}^t f\big(\mb{d};\mb{s}^{(\tau)}\big)\Big) \approx F(\mb{d}) = \E_{\mb{s}}[ f(\mb{d}; \mb{s})] \;.
\end{equation}

The central limit theorem tells us that $F^{(t)}\to F$ as $t\to\infty$. However, as discussed in Section \ref{sec:framework}, $F^{(t)}$ is not computationally tractable. To update $\mb{d}$ efficiently, we introduce the \emph{surrogate function} $\F^{(t)}$ of $F^{(t)}$.  Given $\mb{s}^{(t)}$, $\mb{x}^{(t)}$ is computed by CBPDN (\ref{eqn:cbpdn}) using the latest dictionary $\mb{d}^{(t-1)}$, then a surrogate of $f(\mb{d};\mb{s}^{(t)})$ is given as
\begin{equation}
\label{eqn:loss_surrogate}
\mb{x}^{(t)} = \argmin_{\mb{x}} \ell(\mb{d}^{(t-1)},\mb{x};\mb{s}^{(t)}), \qquad f^{(t)}(\mb{d})\triangleq l\Big(\mb{d}, \mb{x}^{(t)}; \mb{s}^{(t)}\Big) \;,
\end{equation}
The surrogate function of $F^{(t)}$ is defined as
\begin{equation}
\label{eqn:surrogate}
\F^{(t)}(\mb{d}) = \frac{1}{t}\Big(f^{(1)}(\mb{d}) + \cdots + f^{(t)}(\mb{d}) \Big)\;.
\end{equation}
Then, at the $t^{\text{th}}$ step, the dictionary is updated as
\begin{equation}
\label{eqn:d-update}
\mb{d}^{(t)} = \argmin_{\mb{d}\in\R^{ML}} \F^{(t)}(\mb{d}) + \iota_{\mathrm{C}}(\mb{d})  \;.
\end{equation}

\textbf{Solving subproblem (\ref{eqn:d-update}).}\label{para:fista} To solve (\ref{eqn:d-update}), we apply Fast Iterative Shrinkage-Thresholding
(FISTA)~\cite{beck2009fast}, which needs to compute a gradient at each step. The gradient for the surrogate function can be computed as
\[
\nabla \F^{(t)}(\mb{d}) = \frac{1}{t} \Big(\sum_{\tau=1}^t  \big(X^{(\tau)}\big)^T X^{(\tau)}\Big) \mb{d} - \frac{1}{t}\Big(\sum_{\tau=1}^t  \big(X^{(\tau)}\big)^T \mb{s}^{(\tau)}\Big) \;.
\]
We cannot follow this formula directly since the cost increases linearly in $t$.
Instead we perform the recursive updates
\begin{equation}
\label{eqn:A}
A^{(t)} = A^{(t-1)} + (X^{(t)})^T X^{(t)} \;, \quad \mb{b}^{(t)} = \mb{b}^{(t-1)} + (X^{(t)})^T \mb{s}^{(t)} \;,
\end{equation}
where $(X^{(t)})^T X^{(t)}$ is the Hessian matrix of $f^{(t)}$. These updates, which have a constant cost per step, yield $\nabla \F^{(t)}(\mb{d}) = (A^{(t)}\mb{d} - \mb{b}^{(t)})/t$. The matrix $A^{(t)}/t$, the Hessian matrix of the surrogate function $\F^{(t)}$, accumulates the Hessian matrices of all the past loss functions. This is why we call this method the \emph{second-order stochastic approximation} method.

\textbf{Practical issues}
\begin{itemize}
\item Inaccurate loss function: The surrogate function $\F^{(t)}$ involves old loss functions $f^{(1)}, f^{(2)}, \cdots$, which contain old information $\mb{x}^{(1)}, \mb{x}^{(2)},\cdots$. For example, $\mb{x}^{(1)}$ is computed using $\mb{d}^{(0)}$ (cf.  (\ref{eqn:loss_surrogate})).
\item Large single step complexity and memory usage: handling a whole image $\mb{s}^{(t)}$ at each time is still a large-scale problem.
\item FISTA is slow at solving subproblem (\ref{eqn:d-update}): FISTA takes many steps to reach a sufficient accuracy.
\end{itemize}
To address these points, four modifications are given in this section\footnote{Improvements I and II have been addressed in our previous work \cite{liu-2017-online}. In the present article, we include their theoretical analysis and introduce the new enhancement of the stopping criterion (Improvement III).}.

\subsection{Improvement I: forgetting factor}
\label{para:modify}

At time $t$, the dictionary is the result of an accumulation of past coefficient maps $\mb{x}^{(\tau)}_m$, $\tau < t$, which were computed with the then-available dictionaries. A way to balance accumulated past contributions and the information provided by the new training samples is to compute a weighted combination of these contributions~\cite{skretting-2010-recursive, mairal-2010-online, szabo2011online, slavakis2014online}. This combination gives more weight to more recent updates since those are the result of a more extensively trained dictionary. Specifically, we consider the following weighted (or modified) surrogate function:
\begin{equation}
\label{eqn:surrogate_mod_explicit}
\F^{(t)}_{\mathrm{mod}} (\mb{d})= \frac{1}{ \Lambda^{(t)}} \sum_{\tau=1}^t  (\tau/t)^p f^{(\tau)}(\mb{d})\;, \quad \Lambda^{(t)} = \sum_{\tau=1}^t (\tau/t)^p \;.
\end{equation}
This function can be written in recursive form as
\begin{align}
\Lambda^{(t)} =& \alpha^{(t)} \Lambda^{(t-1)} + 1 \label{eqn:Lambda} \;,\\
\Lambda^{(t)}\F^{(t)}_{\mathrm{mod}} (\mb{d}) =& \alpha^{(t)} \Lambda^{(t-1)} \F^{(t-1)}_{\mathrm{mod}}(\mb{d}) + f^{(t)}(\mb{d}) \label{eqn:surrogate_mod} \;.
\end{align}
Here $\alpha^{(t)}\in (0,1)$ is a \emph{forgetting factor}, which has its own time evolution:
\vspace{-1mm}
\begin{equation}
\label{eqn:forget}
\alpha^{(t)} = (1-1/t)^p
\vspace{-1mm}
\end{equation}
regulated by the \emph{forgetting exponent} $p>0$. As $t$ increases, the factor $\alpha^{(t)}$ increases ($\alpha^{(t)} \to 1$ as $t \to \infty$), reflecting the increasing accuracy of the past information as the training progresses.
The dictionary update  (\ref{eqn:d-update}) is  modified correspondingly to
\vspace{-2mm}
\begin{equation}
\label{eqn:d-update-mod}
\mb{d}^{(t)} = \argmin_{\mb{d}\in\R^{ML}} \F^{(t)}_{\mathrm{mod}}(\mb{d}) + \iota_{\mathrm{C}}(\mb{d}) \;.
\vspace{-2mm}
\end{equation}

This technique has been used in some previous dictionary learning works, as we mentioned before, but was not theoretically analyzed. In this paper, we prove in Propositions \ref{lemma:weight_clt} and \ref{lemma:weight} that $F^{(t)}_{\mathrm{mod}} \to F$ as $t\to\infty$, where $F^{(t)}_{\mathrm{mod}}$ is a weighted approximation of $F$:
\begin{equation}
\label{eqn:F}
F^{(t)}_{\mathrm{mod}}(\mb{d}) = \frac{1}{\Lambda^{(t)}}\Big(\sum_{\tau=1}^t (\tau/t)^p f(\mb{d};\mb{s}^{(\tau)})\Big) \;.
\end{equation}
Moreover, in Theorem \ref{prop:main}, $\F^{(t)}_{\mathrm{mod}}$, the surrogate of $F^{(t)}_{\mathrm{mod}}$, is also proved to be convergent on the current dictionary, \ie $\F^{(t)}_{\mathrm{mod}}(\mb{d}^{(t)}) - F^{(t)}_{\mathrm{mod}}(\mb{d}^{(t)}) \to 0$.

\textbf{Effect of the forgetting exponent $p$.} A small $p$ tends to lead to a stable algorithm since all the training signals are given nearly equal weights and $F^{(t)}_{\mathrm{mod}}$ is a stochastic approximation of $F$ with small variance. Propositions \ref{lemma:weight_clt} and \ref{lemma:weight} give theoretical explanations of this phenomenon.
However, a small $p$ leads to an inaccurate surrogate loss function $\F^{(t)}_{\mathrm{mod}}$ since it gives large weights to old information. In the extreme case, as $p\to0$, the modified surrogate function (\ref{eqn:surrogate_mod_explicit}) reduces to the standard one (\ref{eqn:surrogate}). Section \ref{sec:compare_p} reports the related numerical results.

\subsection{Improvement II: image-splitting}
\label{sec:region-sample}

Both the single-step complexity and memory usage are related to the signal dimension $N$. For a typical imaging problem, $N = 256\times256$ or greater, which is large. To reduce the complexities, we use small regions \footnote{In our previous work \cite{liu-2017-online}, we sample some small regions from the whole signals in the limited memory algorithm, which performs worse than the algorithm training with the whole signals. We claimed that the performance sacrifice is caused by the circular boundary condition. In fact, this is caused by the sampling. In that paper, we sample small regions with random center position and fixed size. If we sample small regions in this way, some parts of the image are not sampled, but some are sampled several times.
Consequently, in the present paper, we propose the ``image-splitting'' technique in Algorithm \ref{algo:surro-splitting}, which avoids this issue. It only shows worse performance when the splitting size is smaller than a threshold, which is actually caused by the boundary condition.}
instead of the whole signal. Specifically, as illustrated in~\fig{regionsample}, we split a signal $\mb{s}^{(t)} \in N$ into small regions ${\mb{s}^{(t)}_{\text{split},1}, \mb{s}^{(t)}_{\text{split},2}, ... } \in \tilde{N}$, with $\tilde{N} < N$, and treat them as if they were distinct signals.  In this way, the training signal sequence becomes
\vspace{-2mm}
\[
\{\mb{s}_{\text{split}}\} \triangleq \{ \mb{s}^{(1)}_{\text{split},1}, \cdots,\mb{s}^{(1)}_{\text{split},n}, \mb{s}^{(2)}_{\text{split},1}, \cdots, \mb{s}^{(2)}_{\text{split},n}, \cdots\} \;.
\vspace{-1mm}
\]

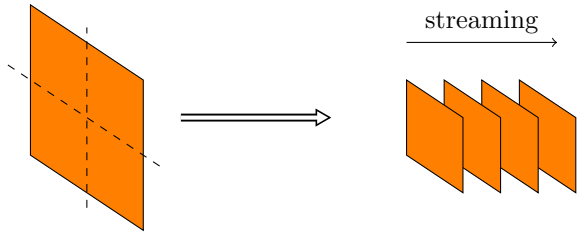
\begin{figure}[t]
\centering
\tikzstyle{vecArrow} = [thick, decoration={markings,mark=at position
   1 with {\arrow[semithick]{open triangle 60}}},
   double distance=1.4pt, shorten >= 5.5pt,
   preaction = {decorate},
   postaction = {draw,line width=1.4pt, white,shorten >= 4.5pt}]

\begin{tikzpicture}






\shade[top color = orange, bottom color = orange] ((0,0) -- (0,2) -- (1.5, 1) -- (1.5,-1) --cycle;
\draw[black, thin] (0,0) -- (0,2) -- (1.5, 1) -- (1.5,-1) --cycle;

\draw[black, dashed] (-0.3,1.2) -- (1.8,-0.2);
\draw[black, dashed] (0.75,-0.7) -- (0.75,1.7);

\shade[top color = orange, bottom color = orange] (6.5,0) -- (6.5,1) -- (7.25, 0.5) -- (7.25,-0.5) --cycle;
\draw[black,thin] (6.5,0) -- (6.5,1) -- (7.25, 0.5) -- (7.25,-0.5) --cycle;

\shade[top color = orange, bottom color = orange] (6,0) -- (6,1) -- (6.75, 0.5) -- (6.75,-0.5) --cycle;
\draw[black,thin] (6,0) -- (6,1) -- (6.75, 0.5) -- (6.75,-0.5) --cycle;

\shade[top color = orange, bottom color = orange] (5.5,0) -- (5.5,1) -- (6.25, 0.5) -- (6.25,-0.5) --cycle;
\draw[black,thin] (5.5,0) -- (5.5,1) -- (6.25, 0.5) -- (6.25,-0.5) --cycle;

\shade[top color = orange, bottom color = orange] (5,0) -- (5,1) -- (5.75, 0.5) -- (5.75,-0.5) --cycle;
\draw[black,thin] (5,0) -- (5,1) -- (5.75, 0.5) -- (5.75,-0.5) --cycle;

\draw[->] (5,1.5) -- (7,1.5) node[above, midway] {streaming};

   \draw[vecArrow] (2,0.5) to (4, 0.5);

\end{tikzpicture}
\caption{An example of image splitting: $N=256\times256 \to \tilde{N}=128\times128$. }
\label{fig:regionsample}
\end{figure}

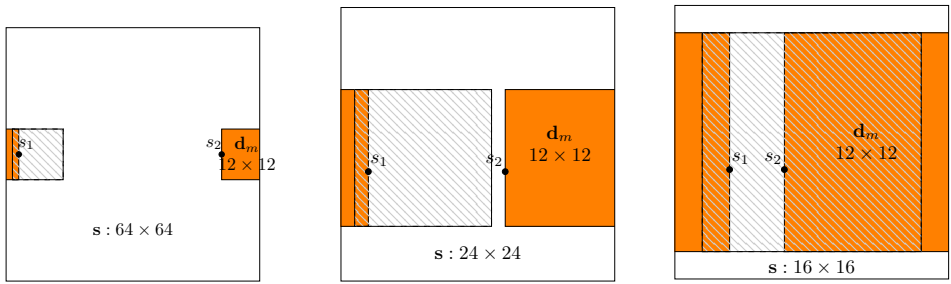
\begin{figure}[t]
    \centering
\subfigure[When the signal size $64\times64$ is much larger than the kernel size $12\times12$, pixels $s_1,s_2$ in the same filter are far from each other. Thus, they do not interact with each other.]{
      \resizebox{0.3\textwidth}{!}{ \definecolor{mygray}{gray}{0.8}
\begin{tikzpicture}

\shade[top color = orange, bottom color = orange] (5,2) -- (4.25,2) -- (4.25, 3) -- (5,3) --cycle;
\shade[top color = orange, bottom color = orange] (0,2) -- (0.25,2) -- (0.25, 3) -- (0,3) --cycle;

\draw[black,thin] (0,0) -- (0,5) -- (5, 5) -- (5,0) --cycle;

\draw[black,thin] (5,2) -- (4.25,2) -- (4.25, 3) -- (5,3) ;

\draw[black,thin] (0,2) -- (0.25,2) -- (0.25, 3) -- (0,3) ;

\draw[dashed,thin] (0.125,2) -- (1.125,2) -- (1.125, 3) -- (0.125,3) --cycle;
\draw[pattern=north west lines, pattern color=mygray] (0.125,2) rectangle (1.125, 3);


\node[circle,draw=black, fill=black, inner sep=0pt,minimum size=3pt] (b) at (0.25,2.5) {};
\node[circle,draw=black, fill=black, inner sep=0pt,minimum size=3pt] (b) at (4.25,2.5) {};

\begin{scope}[ execute at begin node = $\displaystyle
                 , execute at end node   = $]
      \node at (2.5,1) {\mb{s}: 64 \times 64} ;
	  \node at (4.75,2.7) {\mb{d}_m} ;
	  \node at (4.75,2.3) {12 \times 12};
	  \node at (0.4,2.7) {s_1};
	  \node at (4.1,2.7) {s_2};
    \end{scope}
\end{tikzpicture}\label{fig:boundary1}}}
    \hfill
\subfigure[When the signal size $24\times24$ is twice the kernel size $12\times12$, $s_1,s_2$ still do not interact. It is the smallest signal size to avoid boundary artifacts.]{
      \resizebox{0.3\textwidth}{!}{ \definecolor{mygray}{gray}{0.8}
\begin{tikzpicture}

\shade[top color = orange, bottom color = orange] (5,1) -- (3,1) -- (3, 3.5) -- (5,3.5) --cycle;
\shade[top color = orange, bottom color = orange] (0,1) -- (0.5,1) -- (0.5, 3.5) -- (0,3.5) --cycle;

\draw[black,thin] (0,0) -- (0,5) -- (5, 5) -- (5,0) --cycle;

\draw[black,thin] (5,1) -- (3,1) -- (3, 3.5) -- (5,3.5);

\draw[black,thin] (0,1) -- (0.5,1) -- (0.5, 3.5) -- (0,3.5) ;

\draw[dashed,thin] (0.25,1) -- (2.75,1) -- (2.75, 3.5) -- (0.25,3.5) --cycle;
\draw[pattern=north west lines, pattern color=mygray] (0.25,1) rectangle (2.75, 3.5);

\node[circle,draw=black, fill=black, inner sep=0pt,minimum size=3pt] (b) at (0.5,2) {};
\node[circle,draw=black, fill=black, inner sep=0pt,minimum size=3pt] (b) at (3,2) {};

\begin{scope}[ execute at begin node = $\displaystyle
                 , execute at end node   = $]
      \node at (2.5,0.5) {\mb{s}: 24 \times 24} ;
	  \node at (4,2.7) {\mb{d}_m} ;
	  \node at (4,2.3) {12 \times 12};
	  \node at (0.7,2.2) {s_1};
	  \node at (2.8,2.2) {s_2};
    \end{scope}
\end{tikzpicture}\label{fig:boundary2}}}
    \hfill
\subfigure[When the signal size $16\times16$ is less than twice the kernel size $12\times12$, $s_1,s_2$ interact with one another. This leads to artifacts in practice.]{
      \resizebox{0.3\textwidth}{!}{ \definecolor{mygray}{gray}{0.8}
\begin{tikzpicture}

\shade[top color = orange, bottom color = orange] (5,0.5) -- (2,0.5) -- (2, 4.5) -- (5,4.5) --cycle;
\shade[top color = orange, bottom color = orange] (0,0.5) -- (1,0.5) -- (1, 4.5) -- (0,4.5) --cycle;

\draw[black,thin] (0,0) -- (0,5) -- (5, 5) -- (5,0) --cycle;

\draw[black,thin] (5,0.5) -- (2,0.5) -- (2, 4.5) -- (5,4.5) ;

\draw[black,thin] (0,0.5) -- (1,0.5) -- (1, 4.5) -- (0,4.5) ;

\draw[dashed,thin] (0.5,0.5) -- (4.5,0.5) -- (4.5, 4.5) -- (0.5,4.5) --cycle;
\draw[pattern=north west lines, pattern color=mygray] (0.5, 0.5) rectangle (4.5, 4.5);

\node[circle,draw=black, fill=black, inner sep=0pt,minimum size=3pt] (b) at (1,2) {};
\node[circle,draw=black, fill=black, inner sep=0pt,minimum size=3pt] (b) at (2,2) {};

\begin{scope}[ execute at begin node = $\displaystyle
                 , execute at end node   = $]
      \node at (2.5,0.2) {\mb{s}: 16 \times 16} ;
	  \node at (3.5, 2.7) {\mb{d}_m} ;
	  \node at (3.5, 2.3) {12 \times 12};
	  \node at (1.2,2.2) {s_1};
	  \node at (1.8,2.2) {s_2};
    \end{scope}
\end{tikzpicture}\label{fig:boundary3}}}
    \caption{An illustration of the boundary artifacts
with two-dimensional square signals and dictionary kernels.
}
 \label{fig:boundary}
\end{figure}

\textbf{Boundary issues.} The use of \emph{circular boundary conditions} for signals that are not periodic has the potential to introduce boundary artifacts in the representation, and therefore also in the learned dictionary~\cite{zeiler-2010-deconvolutional}. When the size of the training images is much larger than the kernels, there is some evidence that the effect on the learned dictionary is negligible ~\cite{bristow-2013-fast}, but it is reasonable to expect that these effects will become more pronounced for smaller training images, such as the regions we obtain when using a small splitting size $\tilde{N}$. The possibility of severe artifacts when the image size approaches the kernel size is illustrated in~\fig{boundary}.
In~\sctn{compare_n}, we study this effect and show that using a splitting size that is twice the kernel size
in each dimension
is sufficient to avoid artifacts, as expected from the argument illustrated in~\fig{boundary}.

\subsection{Improvement III: stopping FISTA early}
\label{para:inexact}

Another issue in surrogate function method is the stopping condition of FISTA.  A small fixed tolerance will result in too many inner-loop iterations for the initial steps. Another strategy, as used in SPAMS \cite{mairal2009online, jenatton2010proximal} is a fixed number of inner-loop iterations, but it does not have any theoretical convergence guarantee.

In this article, we propose a ``diminishing tolerance'' scheme in which subproblem (\ref{eqn:d-update-mod}) is solved \emph{inexactly}, but the online learning algorithm is still theoretically guaranteed to converge.
The stopping accuracy is increasing as $t$ increases.
Specifically, the stopping tolerance is decreased as $t$ increases. Moreover, with a warm start (using $\mb{d}^{(t-1)}$ as the initial solution for the $t^{\text{th}}$ step), the number of inner-loop iterations stays moderate as $t$ increases, which is validated by the results in \fig{fix_vs_tol}.

\textbf{Stopping metric.} We use the Fixed Point Residual (FPR) \cite{davis2016convergence}
\begin{equation}
\label{eq:fpr}
\fpr^{(t)}(\mb{g}) \triangleq \normsz[\Big]{ \mb{g} - \text{Proj}_{\text{C}}\big(\mb{g}-\eta\nabla \F_{\mathrm{mod}}^{(t)}(\mb{g})\big)} \;.
\end{equation}
for two reasons. One is its simplicity; if FISTA is used to solve (\ref{eqn:d-update-mod}), this metric can be computed directly as $\fpr^{(t)}(\mb{g}^{j}_{\text{aux}}) = \norm{\mb{g}^{j+1}-\mb{g}^j_{\text{aux}}}$. The other is that a small FPR implies a small distance to the exact solution of the subproblem, as shown in Proposition \ref{lemma:fista} below.

\textbf{Stopping condition.} In this paper, we consider the following stopping condition:
\begin{equation}
\label{eqn:stop_condition}
\fpr^{(t)}(\mb{g}^{j}_{\text{aux}}) \leq \tau^{(t)} \triangleq \tau_0 / (1 + \alpha t) \;,
\end{equation}
where the tolerance $\tau^{(t)}$ is large during the first several steps and reduces to zeros at the rate of $\co(1/t)$ as $t$ increases.  In the $t^{\text{th}}$ step, once (\ref{eqn:stop_condition}) is satisfied, we stop the $D$-update (FISTA) and continue to the next step. The effect of this stopping condition is theoretically analyzed in Propositions \ref{lemma:fista} and \ref{lemma:d_residual}, and numerically demonstrated in~\sctn{fix-vs-diminish} below.

\subsection{Improvement IV: computational techniques in solving subproblem (\ref{eqn:d-update-mod})}
\label{para:freq}

 Based on the discussion in Section \ref{sec:freq}, we have two options to solve subproblem (\ref{eqn:d-update-mod}). One is to solve in the spatial domain utilizing sparsity. The gradient of $\F^{(t)}_{\mathrm{mod}}(\mb{d}) $ is
\[
\nabla \F^{(t)}_{\mathrm{mod}}(\mb{d}) = \frac{1}{\Lambda^{(t)}} \sum_{\tau=1}^t  (\tau/t)^p \Big((X^{(t)})^T X^{(t)} \mb{d} - (X^{(\tau)})^T \mb{s}^{(\tau)}\Big) =  \frac{1}{\Lambda^{(t)}} \Big(A^{(t)}_{\mathrm{mod}} \mb{d} - \mb{b}^{(t)}_{\mathrm{mod}}\Big) \;,
\]
where $A^{(t)}_{\mathrm{mod}}$ and $\mb{b}^{(t)}_{\mathrm{mod}}$ are calculated in a recursive form in the line 5 of Algorithm \ref{algo:surro-splitting}. The other option is to update in the frequency domain. The conjugate cogradient of $\hat{\F}^{(t)}_{\mathrm{mod}}(\hat{\mb{d}}) $ is
\[
\nabla \hat{\F}^{(t)}_{\mathrm{mod}}(\hat{\mb{d}}) = \frac{1}{\Lambda^{(t)}} \sum_{\tau=1}^t  (\tau/t)^p \Big((\hat{X}^{(t)})^T \hat{X}^{(t)} \hat{\mb{d}} - (\hat{X}^{(\tau)})^T \hat{\mb{s}}^{(\tau)}\Big) =  \frac{1}{\Lambda^{(t)}} \Big(\hat{A}^{(t)}_{\mathrm{mod}} \hat{\mb{d}} - \hat{\mb{b}}^{(t)}_{\mathrm{mod}}\Big) \;,
\]
where $\hat{A}^{(t)}_{\mathrm{mod}}$ and $\hat{\mb{b}}^{(t)}_{\mathrm{mod}}$ are calculated in a recursive form in the line 7 of Algorithm \ref{algo:surro-splitting}.
With the gradients, we can apply FISTA or frequency-domain FISTA on the problem (\ref{eqn:d-update-mod}), as in Algorithm \ref{algo:surro-splitting}.

\begin{algorithm2e}[h]
\SetKwInOut{initial}{Initialize}
\initial{Initialize $\mb{d}^{(0)}$, let $A^{(0)}_{\mathrm{mod}} \leftarrow 0, \mb{b}^{(0)}_{\mathrm{mod}} \leftarrow 0$ or $\hat{A}^{(0)}_{\mathrm{mod}} \leftarrow 0, \hat{\mb{b}}^{(0)}_{\mathrm{mod}} \leftarrow 0$. }
\For{$t=1,\cdots, T$} {
Sample a signal $\mb{s}^{(t)}$ from $\{\mb{s_{\text{split}}}\}$.\\
Solve convolutional sparse coding problem (\ref{eqn:cbpdn_t}) to obtain $\mb{x}^{(t)}$.\\
\uIf{Option I}
{Update $A^{(t)}_{\mathrm{mod}}, \mb{b}^{(t)}_{\mathrm{mod}}$ in the spatial-domain with sparse matrix $X^{(t)}$:
\vspace{-2mm}\[A^{(t)}_{\mathrm{mod}} = \alpha^{(t)} A^{(t-1)}_{\mathrm{mod}} + (X^{(t)})^T X^{(t)} ,~
\mb{b}^{(t)}_{\mathrm{mod}} = \alpha^{(t)}\mb{b}^{(t-1)}_{\mathrm{mod}} + (X^{(t)})^T \mb{s}^{(t)}
\vspace{-1mm}
\]

Solve the following subproblem with FISTA (stopping condition (\ref{eqn:stop_condition})):
\vspace{-2mm}
\[
\mb{d}^{(t)} = \argmin_{\mb{d}\in\R^{ML}} \F^{(t)}_{\mathrm{mod}}(\mb{d}) + \iota_{\mathrm{C}}(\mb{d}) \;.
\vspace{-2mm}
\]
}
\ElseIf{Option II}{
Update $\hat{A}^{(t)}_{\mathrm{mod}}, \hat{\mb{b}}^{(t)}_{\mathrm{mod}}$ in the frequency-domain:
\vspace{-2mm}\[\hat{A}^{(t)}_{\mathrm{mod}} = \alpha^{(t)} \hat{A}^{(t-1)}_{\mathrm{mod}} + (\hat{X}^{(t)})^H \hat{X}^{(t)}, ~
\hat{\mb{b}}^{(t)}_{\mathrm{mod}} = \alpha^{(t)}\hat{\mb{b}}^{(t-1)}_{\mathrm{mod}} + (\hat{X}^{(t)})^H \hat{\mb{s}}^{(t)}\vspace{-3mm}\]

Solve the following subproblem with frequency-domain FISTA (stopping condition (\ref{eqn:stop_condition}), see Appendix \ref{sec:fista_freq}):
\vspace{-2mm}
\[
\mb{d}^{(t)} = \argmin_{\mb{d}\in\R^{MN}} \hat{\F}^{(t)}_{\mathrm{mod}}(\hat{\mb{d}}) + \iota_{\text{C}_{\text{PN}}}(\mb{d}) \;.
\vspace{-6mm}
\]
}
}
\KwOut{$\mb{d}^{(T)}$}
\caption{Online Convolutional Dictionary Learning (Surrogate-Splitting)}
\label{algo:surro-splitting}
\end{algorithm2e}

\textbf{Complexity analysis of Algorithm \ref{algo:surro-splitting}.}
If we solve (\ref{eqn:d-update-mod}) directly, the operator $X^{(t)}$ is a linear operator from $\R^{LM}$ to $\R^{N}$. Thus, the complexity of computing the Hessian matrix of $f^{(t)}$, $(X^{(t)})^T X^{(t)}$, is $\co(L^2M^2N)$ and the memory cost is $\co(L^2M^2)$.
Otherwise, if we solve (\ref{eqn:d-update-mod}) utilizing the sparsity of $X$, the computational cost of computing $(X^{(t)})^T X^{(t)}$ can be reduced to $\co(L^2M^2N\rho)$, where $\rho$ is the density of sparse matrix $X^{(t)}$, but the memory cost is still $\co(L^2M^2)$ because $(X^{(t)})^T X^{(t)}$ is not sparse although $X^{(t)}$ is.
In comparison, if we solve (\ref{eqn:d-update-mod}) in the frequency domain, the frequency-domain operator $\hat{X}^{(t)} = (\hat{X}_1, \hat{X}_2, \cdots, \hat{X}_M)$ is a linear operator from $\C^{MN}$ to $\C^{N}$, which seems to lead to a larger complexity to compute the Hessian: $\co(M^2N^3)$ flops and $\co(M^2N^2)$ memory cost. However, since each component $\hat{X}_m$ is diagonal, the frequency-domain product $(\hat{X}^{(t)})^H \hat{X}^{(t)}$ has only $\co(M^2N)$ non-zero values. Both the number of flops and memory cost are $\co(M^2N)$.  The complexities are listed in Table \ref{tab:surro}.

\begin{table}[t]
\centering
\begin{tabular}{|b{0.268\linewidth}|b{0.48\linewidth}|b{0.268\linewidth}|}
\hline
Scheme & Single step complexity & Memory usage\\
\hline \hline
Spatial (dense matrix) & $T_{\text{CBPDN}} + \co(L^2M^2N) + J \times \co(L^2M^2)$ & $\co(L^2M^2) + \co(LMN)$\\
\hline
Spatial (sparse matrix) & $T_{\text{CBPDN}} +\co(L^2M^2N\rho) + J \times \co(L^2M^2)$ & $\co(L^2M^2) + \co(LMN\rho)$\\
\hline
Frequency update & $T_{\text{CBPDN}} + \big(\co(M^2N) + \co(MN\log(N)) \big) + J \times \big(\co(M^2N)+\co(MN\log(N))\big)$ & $\co(M^2N)$\\
\hline
\end{tabular}
\caption{Single step complexity and memory usage of Algorithm \ref{algo:surro-splitting}. $N$: signal dimension; $M$: number of dictionary kernels; $L$: size of each kernel; $\rho$: average density of the coefficient maps; $J$: average loops of FISTA in each step. This is numerically tested in Table \ref{tab:spatial_vs_freq}.
}
\label{tab:surro}
\end{table}

\subsection{Convergence of Algorithm \ref{algo:surro-splitting}}

First, we start with some assumptions\footnote{The specific formulas for Assumptions \ref{assume:uniqueness} and \ref{assume:surro} are shown in Appendix \ref{app:assume-details}.}:
\begin{assume}
\label{assume:bdd_signal}
All the signals are drawn from a distribution with a compact support.
\end{assume}
\begin{assume}
\label{assume:uniqueness}
Each sparse coding step (\ref{eqn:cbpdn}) has a unique solution.
\end{assume}
\begin{assume}
\label{assume:surro}
The surrogate functions are strongly convex.
\end{assume}

Assumption \ref{assume:bdd_signal} can easily be guaranteed by normalizing each training signal.  Assumption \ref{assume:uniqueness} is a common assumption in dictionary learning and other linear regression papers \cite{mairal2009online, efron2004least}. Practically, it must be guaranteed by choosing a sufficiently large penalty parameter $\lambda$ in (\ref{eqn:cbpdn}), because a larger penalty parameter leads to a sparser $\mb{x}$.  See Appendix \ref{app:assume-details} for details. Assumption \ref{assume:surro} is a common assumption in RLS (see Definition (3.1) in \cite{johnstone1982exponential}) and dictionary learning (see Assumption B in \cite{mairal-2010-online}).

\begin{prop}[Weighted central limit theorem]
\label{lemma:weight_clt}
Suppose $Z_i \overset{\text{i.i.d}}{\sim}P_Z(z)$, with a compact support, expectation $\mu$, and variance $\sigma^2$. Define the weighted approximation of $Z$: $\hat{Z}_{\mathrm{mod}}^n\triangleq\frac{1}{\sum_{i=1}^n(i/n)^p}\sum_{i=1}^n (i/n)^p Z_i$. Then, we have
\begin{equation}\label{eqn:weight_clt}
\sqrt{n} (\hat{Z}^n_{\mathrm{mod}} - \mu)  \overset{\text{d}}{\to} N\Big(0,\frac{p+1}{\sqrt{2p+1}}\sigma\Big) \;.
\end{equation}
\begin{equation}
\label{eqn:weight_clt_bdd}
\E\Big[\sqrt{n}\big|\hat{Z}^n_{\mathrm{mod}} - \mu\big|\Big] = \co(1) \;.
\end{equation}
\end{prop}
This proposition is an extension of the central limit theorem (CLT). As $p\to0$, it reduces to the standard CLT. The proof is given in Appendix \ref{app:weight_clt}.

\begin{prop}[Convergence of functions]
\label{lemma:weight}With Assumptions \ref{assume:bdd_signal}-\ref{assume:surro}, we have
\begin{equation}
\label{eqn:donsker}
\E\Big[\sqrt{t}\big\|F - F^{(t)}\big\|_{\infty}\Big] \leq M \;,
\end{equation}
\begin{equation}
\label{eqn:weight-clt-f}
\E\Big[\sqrt{t}\big\|F - F^{(t)}_{\mathrm{mod}}\big\|_{\infty}\Big] \leq \frac{p+1}{\sqrt{2p+1}}M \;,
\end{equation}
where $M>0$ is some constant unrelated with $t$, and $\|f\|_{\infty} = \sup_{\mb{d}\in\text{C}}\|f(\mb{d})\|$.
\end{prop}

This proposition is an extension of Donsker's theorem (see Lemma 7 in \cite{mairal-2010-online} and Chapter 19 in \cite{van2000asymptotic}). The proof is given in Appendix \ref{app:weight-clt-f}.

Moreover, it shows that weighted approximation $F^{(t)}_{\mathrm{mod}}$ and standard approximation $F^{(t)}$ have the same asymptotic convergence rate $\co(1/\sqrt{t})$. However, the error bound factor $(p+1)/\sqrt{2p+1}$ is a monotone increasing function in $p\geq0$. Thus, a larger $p$ leads to a larger variance and slower convergence of $F^{(t)}_{\mathrm{mod}}$. This explains why we cannot choose $p$ to be too large.

\begin{prop}[Convergence of FPR implies convergence of iterates]
\label{lemma:fista} Let $(\mb{d}^*)^{(t)}$ be the exact minimizer of the $t^{\text{th}}$ subproblem:
\begin{equation}
\label{eqn:exact-t}
(\mb{d}^*)^{(t)} = \argmin_{\mb{d}}\F^{(t)}_{\mathrm{mod}}(\mb{d}) + \iota_{\text{C}}(\mb{d}) \;.
\end{equation}
Let $\mb{d}^{(t)}$ be the solution obtained by the frequency-domain FISTA (Algorithm \ref{algo:fista}) with our proposed stopping condition (\ref{eqn:stop_condition}). Then, we have
\begin{equation}
\label{eqn:fpr_bdd}
\normsz[\big]{\mb{d}^{(t)}-(\mb{d}^*)^{(t)}} \leq \co\left(t^{-1} \right) \;.
\end{equation}
\end{prop}

The proof is given in Appendix \ref{app:lemma-fista}.

\begin{prop}[The convergence rate of Algorithm \ref{algo:surro-splitting}]
\label{lemma:d_residual} Let $\mb{d}^{(t)}$ be the sequence generated by Algorithm \ref{algo:surro-splitting}. Then, we have
\begin{equation}
\label{eqn:d_residual}
\normsz[\big]{\mb{d}^{(t+1)}-\mb{d}^{(t)}} = \co\left(t^{-1}\right).
\end{equation}
\end{prop}

Compared with Lemma 1 in \cite{mairal-2010-online}, which shows the convergence rate of the surrogate function method with exact $D$-update, our Proposition \ref{lemma:d_residual} shows that the inexact $D$-update (\ref{eqn:stop_condition}) shares the same rate. Since our inexact version stops FISTA earlier, it is faster. The proof of this proposition is given in Appendix \ref{app:lemma-d-residual}.

\begin{theorem}[Almost sure convergence of Algorithm \ref{algo:surro-splitting}]
\label{prop:main}
Let $\F^{(t)}_{\mathrm{mod}}$ be the surrogate function sequence, $\mb{d}^{(t)}$ the iterate sequence, both generated by Algorithm \ref{algo:surro-splitting}. Then we have, with probability $1$:
\begin{enumerate}
\item $\F^{(t)}_{\mathrm{mod}}(\mb{d}^{(t)})$ converges.
\item $\F^{(t)}_{\mathrm{mod}}(\mb{d}^{(t)}) - F(\mb{d}^{(t)}) \to 0$. \label{thm:2}
\item $F(\mb{d}^{(t)})$ converges.\label{thm:3}
\item $\mathrm{dist}(\mb{d}^{(t)},V) \to 0$, where $V$ is the set of stationary points of the CDL problem (\ref{eqn:cdl}).\label{thm:4}
\end{enumerate}
\end{theorem}

The proof is given in Appendix \ref{app:mainprop}.

\section{Learning from masked images}
\label{sec:msk}

In this section, we focus on the masked CDL problem (\ref{eqn:batch-cdl-ms}), for which there are no existing online algorithms. Let
\[
l_{\text{mask}}(\mb{d},\mb{x};\mb{s}) \triangleq \frac{1}{2}
    \normsz[\Big]{W \odot \Big(\sum_{m=1}^M \mb{d}_m \ast \mb{x}_{m} - \mb{s}\Big)}_2^2 +
    \lambda  \sum_{m=1}^M \norm{\mb{x}_{m}}_1.
\]
The objective function is defined as
\begin{equation}
\label{eqn:loss-mask}
f_{\text{mask}}(\mb{d}; \mb{s}) \triangleq \min_{\mb{x}} l_{\text{mask}}(\mb{d},\mb{x};\mb{s}) \;,
\end{equation}
allowing (\ref{eqn:batch-cdl-ms}) to be written as a stochastic minimization problem:
\begin{equation}
\label{eqn:mask-cdl}
\min_{\mb{d}} \E_{\mb{s}}[ f_{\text{mask}}(\mb{d}; \mb{s})] + \iota_{\mathrm{C}}(\mb{d}) \;.
\end{equation}

Both Algorithms \ref{algo:sgd} and \ref{algo:surro-splitting} can be applied to the masked CDL problem (\ref{eqn:mask-cdl}). First, we write $l_{\text{mask}}$ in a concise form:
\[
l_{\text{mask}}(\mb{d},\mb{x};\mb{s}) = 1/2\norm{W\odot X \mb{d} - W\odot \mb{s}}_2^2 + \lambda\norm{\mb{x}}_1.
\]
Thus, if we substitute operator $X$ with $W\odot X$, $\mb{s}$ with $W\odot \mb{s}$, and substitute standard CSC with masked CSC~\cite{heide-2015-fast, wohlberg-2016-boundary}, everything is the same as CDL without $W$, then we can apply Algorithms \ref{algo:sgd} and \ref{algo:surro-splitting}, with Option I, on (\ref{eqn:mask-cdl}). The numerical results for masked CDL are reported in Section \ref{sec:msk_reslt}.

A variant of Algorithm \ref{algo:sgd} with Option II is also able to solve  (\ref{eqn:mask-cdl}). First, $l_{\text{mask}}$ on the frequency domain is
\[
\hat{l}_{\text{mask}}(\hat{\mb{d}},\hat{\mb{x}};\hat{\mb{s}}) = 1/2\norm{W\odot \text{IFFT}\big(\hat{X} \hat{\mb{d}} -  \hat{\mb{s}}\big)}_2^2 + \lambda\norm{\mb{x}}_1.
\]
Similarly to the derivation in Section \ref{sec:sgd}, we derive the conjugate cogradient of $\hat{f}_{\text{mask}}$:
\[
\begin{aligned}
\nabla \hat{f}_{\text{mask}}(\hat{\mb{d}}^{(t-1)}; \hat{\mb{s}}^{(t)}) =& \frac{\partial \hat{l}_{\text{mask}}}{\partial \hat{\mb{d}}} \big(\hat{\mb{d}}^{(t-1)},\hat{\mb{x}}^{(t)}; \hat{\mb{s}}^{(t)} \big)\\ =& \big(\hat{X}^{(t)}\big)^H \text{FFT} \Big\{W \odot \text{IFFT}\big(\hat{X}^{(t)}\hat{\mb{d}}^{(t-1)}-\hat{\mb{s}}^{(t)}\big)\Big\} \;.
\end{aligned}
\]
The masked variant of Algorithm \ref{algo:sgd} Option II can be derived from $\nabla \hat{f}_{\text{mask}}$ as
\[
\mb{d}^{(t)} = \text{Proj}_{\text{C}_\text{PN}}\bigg(\text{IFFT}\Big(\hat{\mb{d}}^{(t-1)} - \eta^{(t)}\nabla \hat{f}_{\text{mask}}(\hat{\mb{d}}^{(t-1)}; \hat{\mb{s}}^{(t)}) \Big)\bigg) \;.
\]

Algorithm \ref{algo:surro-splitting} Option II, however, is not easily applied to (\ref{eqn:mask-cdl}). Computing the Hessian matrix in the frequency domain
\[
\begin{aligned}\hat{A}^{(t)}_{\text{mask}} =& \alpha^{(t)} \hat{A}^{(t-1)}_{\text{mask}} + (\hat{X}^{(t)})^H \text{FFT} \big( W \odot \text{IFFT}( \hat{X}^{(t)} )\big)\\ =&  \alpha^{(t)} \hat{A}^{(t-1)}_{\text{mask}} + (\hat{X}^{(t)})^H \text{FFT} \big( W \odot X^{(t)} \big)\end{aligned}
\]
requires an FFT on each column of a matrix $W \odot X^{(t)} \in \R^{MN \times N}$, which is very computational expensive.

\section{Numerical results: learning from clean data set}
\label{sec:rslt}

All the experiments are computed using MATLAB R2016a running on a workstation with 2 Intel Xeon(R) X5650 CPUs clocked at 2.67GHz. Implementations of these algorithms are available in the Matlab version of the SPORCO software library~\cite{wohlberg-2016-sporco}, and will be included in a future release of the Python version of this library. The dictionary size is $12 \times 12 \times 64$, and the signal size is $256 \times 256$. Dictionaries are evaluated by comparing the functional values obtained by computing CBPDN (\ref{eqn:cbpdn}) on the test set. A smaller functional value indicates a better dictionary. Similar methods to evaluate the dictionary are also used in other dictionary learning works~\cite{mairal-2010-online, tang2012self}.   The regularization parameter is chosen as $\lambda=0.1$.

The training set consists of 40 images selected from the MIRFLICKR-1M dataset\footnote{The actual image data contained in this dataset is of very low resolution since the dataset is primarily targeted at image classification tasks. The original images from which those used here were derived were obtained by downloading the original images from Flickr that were used to derive the MIRFLICKR-1M images.} \cite{huiskes-2008-new}, and the test set consists of 20 different images from the same source.
All of the images used were originally of size $512\times512$. To accelerate the experiments, we crop the borders of both the training images and testing images and preserve the central part to yield $256\times256$. The training and testing images are pre-processed by dividing by 255 to rescale the pixel values to the range $[0, 1]$ and highpass filtering\footnote{The pre-processing is applied due to the inability of the standard CSC model to effectively represent low-frequency/large-scale image components~\cite[Sec. 3]{wohlberg-2016-convolutional2}. In this case the highpass component is computed as the difference between the input signal and a lowpass component computed by Tikhonov regularization with a gradient term~\cite[pg. 3]{wohlberg-2017-sporco}, with regularization parameter $5.0$. }.

In this work we solve the convolutional sparse coding step using an ADMM algorithm~\cite{wohlberg-2014-efficient} with an adaptive penalty parameter scheme~\cite{wohlberg2017admm}. The stopping condition is that both primal and dual normalized residuals~\cite{wohlberg2017admm} be less than $10^{-3}$, and the relaxation parameter is set to $1.8$~\cite{wohlberg-2016-efficient}.

\subsection{Validation of Algorithm \ref{algo:sgd}}\label{sec:sgd_rslt}

\begin{figure}[t]
\centering \small
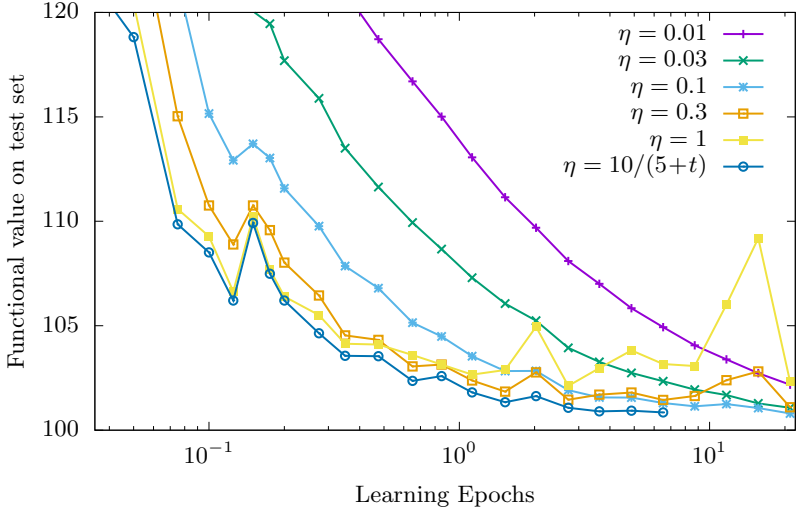
\caption{Tuning the step size of modified SGD (Algorithm \ref{algo:sgd}). A learning epoch is one pass through the whole training set.  With a large fixed $\eta$, the algorithm converges fast at first but becomes unstable later on; with a small fixed $\eta$, the algorithm converges slowly. A diminishing step size provides a good balance.}
\label{fig:sgd_all}
\end{figure}

First we test the effect of step size $\eta^{(t)}$ in Algorithm \ref{algo:sgd}. We can choose either a fixed step size or a diminishing step size:
\vspace{-2mm}
\[
\eta^{(t)} = \eta_0  \quad \text{or} \quad \eta^{(t)}=a/(t+b).
\vspace{-1mm}
\]
The results of experiments to determine the best choice of $\eta$ are reported in~\fig{sgd_all}.  We test the convergence performance of fixed step size scheme with values: $\eta_0 \in \{1,\; 0.3,\; 0.1,\; 0.03,\; 0.01\}$. We also test the convergence performance of the diminishing step size scheme with values: $ a \in \{ 5, 10, 20\}; b \in \{5, 10, 20\}$ and report the best ($a=10,b=5$) in \fig{sgd_all}. When a large fixed step size is used, the functional value decreases fast initially but becomes unstable later on. A smaller step size causes the opposite. A diminishing step size balances accuracy and convergence rate.

Second, we test the computational techniques (computing with sparsity / computing in the frequency domain), as Table \ref{tab:sgd} shows. Both techniques reduce the complexity of updating $\mb{d}^{(t)}$.  Option I has better memory cost while Option II has better calculation time. \fig{sgd} shows the objective values versus training time.

\begin{table}[]
\centering
\begin{tabular}{|c|c|c|c|c|c|}
\hline
\multirow{2}{*}{Schemes}  & \multicolumn{4}{c|}{Average single-step complexity (seconds)}  & \multirow{2}{*}{\begin{tabular}[c]{@{}c@{}}Memory \\ Usage (MB)\end{tabular}} \\ \cline{2-5}
                                   & CBPDN & FFT/IFFT & Update $\mb{d}^{(t)}$  & Total &                                                                               \\ \hline
Spatial (dense matrix)                & 14.8  & 0        & 1.978                 & 16.8 &                                                                              2346.44 \\ \hline
Spatial (sparse matrix)                          & 14.8  & 0        & 0.241                 & 15.1  &                                                                              111.38 \\ \hline
Frequency domain                  & 14.8  & 0.047    & 0.025                 & 14.9  &                                                                              154.84 \\ \hline
\end{tabular}
\caption{Comparison between different options of Algorithm \ref{algo:sgd}. $\lambda=0.1$, average density of $X$: 0.0037. This is the validation of Table \ref{tab:summary}.}
\label{tab:sgd}
\end{table}

\begin{figure}[t]
\centering \small
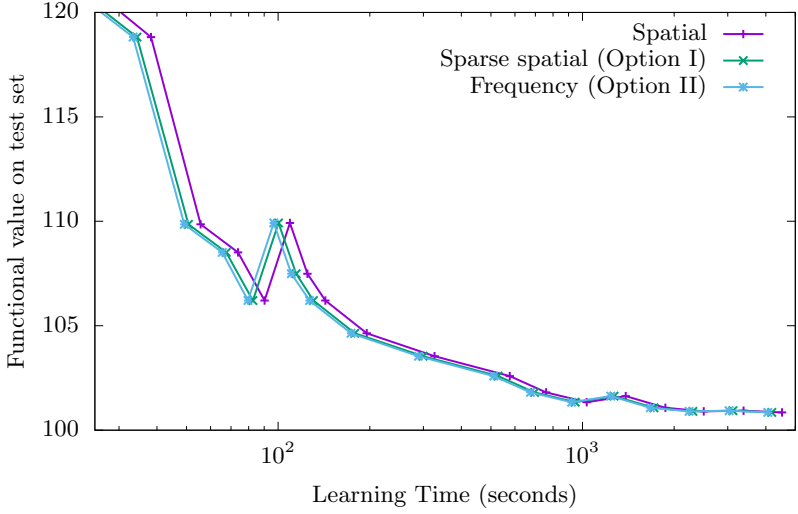
\caption{Different options of Algorithm \ref{algo:sgd}. Frequency-domain update (Option II) performs the best.}
\label{fig:sgd}
\end{figure}

\subsection{Validation of Algorithm \ref{algo:surro-splitting}}

For Algorithm \ref{algo:surro-splitting}, we test the four techniques separately: the forgetting exponent $p$, image splitting with size $\tilde{N}$, and stopping tolerance of FISTA $\tau^{(t)}$, and computational techniques (sparsity or frequency-domain update).

\subsubsection{Validation of Improvement I: forgetting exponent $p$}
\label{sec:compare_p}

\begin{figure}[t]
\centering \small
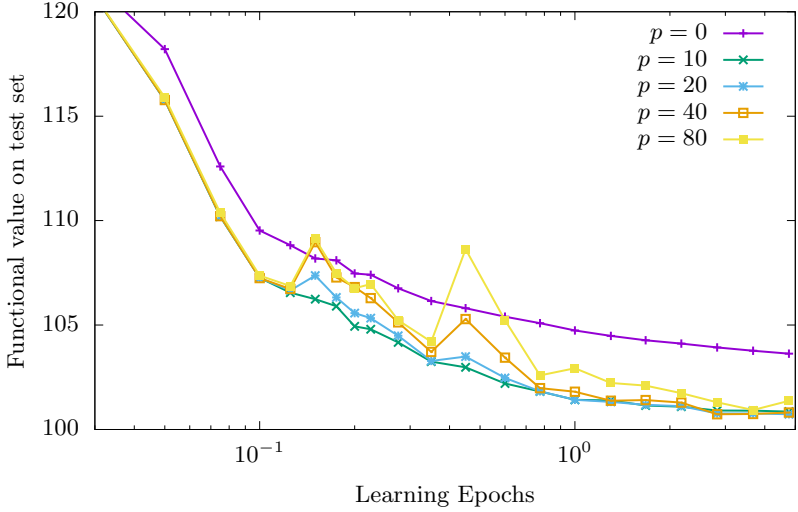
\caption{Effect of Technique I (forgetting exponent $p$) in Algorithm \ref{algo:surro-splitting}. A small $p$ leads to a higher functional value
while a large $p$ leads to instability. $p=10$ is a good choice.}
\label{fig:compare_p}
\end{figure}

In this section, we fix $\tilde{N}=256\times256$ (no splitting) and $\tau^{(t)} = 10^{-4}$, which is small enough to give an accurate solution.  \fig{compare_p} shows that, when $p=0$, the curve is monotonic and with small oscillation, but it converges to a higher functional value. When $p$ is larger, the algorithm converges to lower functional values. When $p$ is too large, for instance, $p \in \{40, 80\}$, the curve oscillates severely, which indicates large variance. These results are consistent with Propositions \ref{lemma:weight_clt} and \ref{lemma:weight}.  In the remaining sections we fix $p=10$ since it is shown to be a good choice.

\subsubsection{Validation of Improvement II: image splitting with size $\tilde{N}$ and boundary artifacts}
\label{sec:compare_n}

\begin{figure}[t]
\centering \small
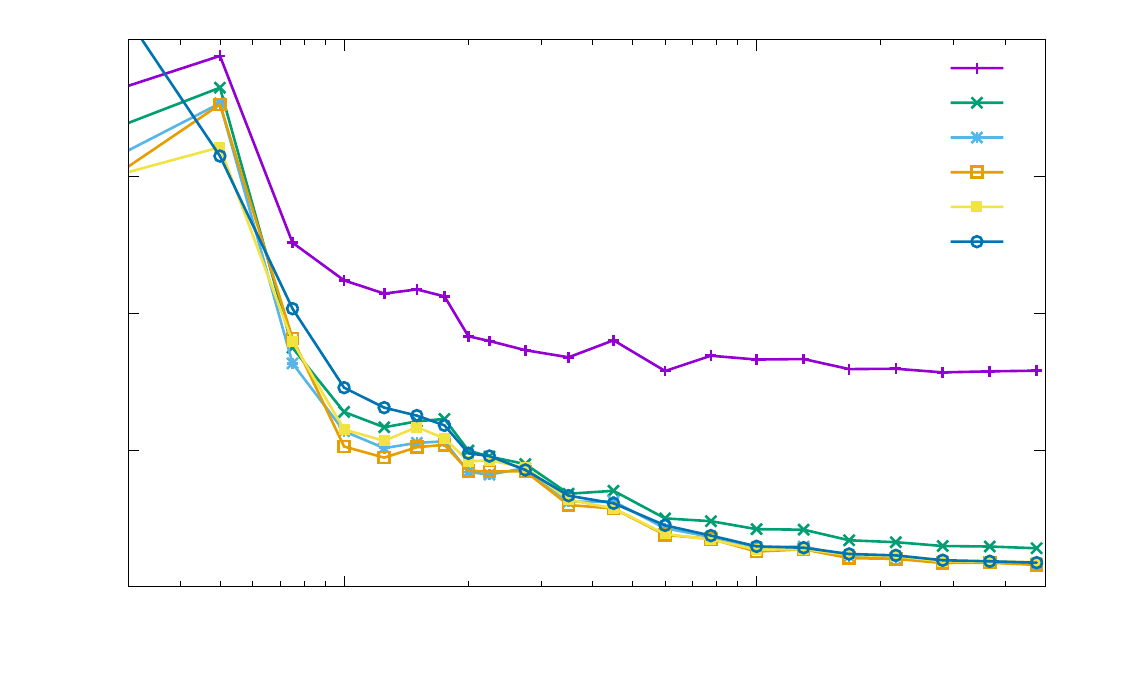
\caption{Effect of the Technique II (image-splitting with size $\tilde{N}$) in Algorithm \ref{algo:surro-splitting}. A learning epoch is one pass through the whole training set. Boundary artifacts become significant when the splitting region size is smaller than twice the dictionary kernel size. Here the kernel size is $12 \times 12$ so the threshold is $24\times24$.}
\label{fig:compare_n}
\end{figure}

\begin{figure}[t]
\centering \small
\begin{tabular}{ccc}
\hspace{-15mm}
\subfigure[][\parbox{0.21\textwidth}{Dictionaries learned by $\tilde{N}=12\times12$: some incomplete features.}]{
       \includegraphics[width=0.43\textwidth]{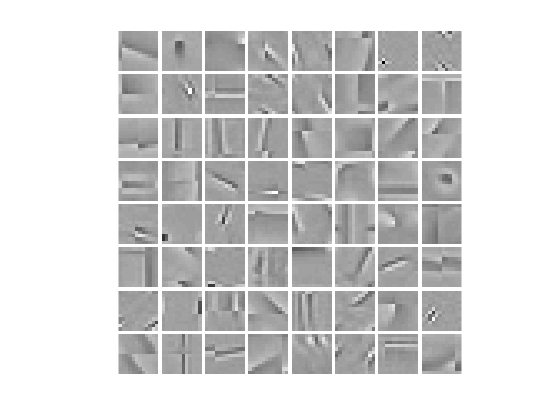}\label{fig:n12}}
&
\hspace{-15mm}
\subfigure[][\parbox{0.21\textwidth}{Dictionaries learned by $\tilde{N}=64\times64$.}]{
       \includegraphics[width=0.43\textwidth]{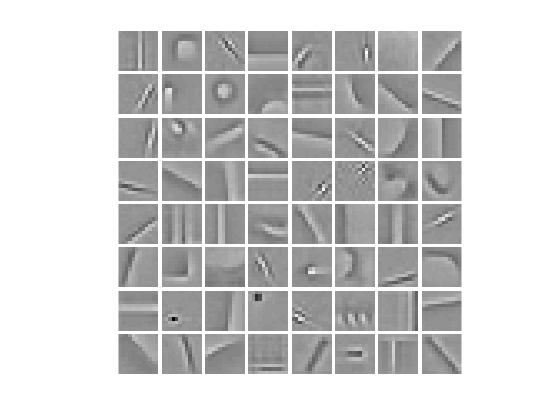}\label{fig:n64}}
&
\hspace{-15mm}
\subfigure[][\parbox{0.21\textwidth}{Dictionaries learned by $\tilde{N}=256\times256$ (no splitting).}]{
       \includegraphics[width=0.43\textwidth]{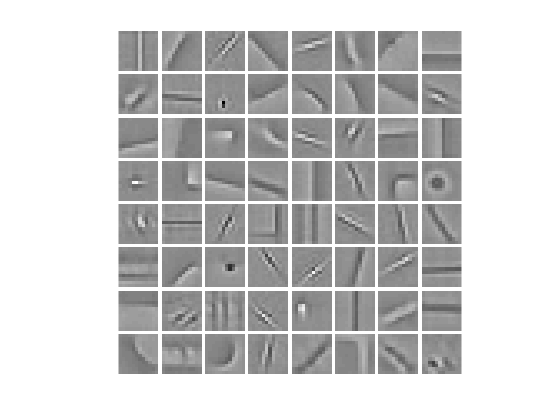}\label{fig:n256}}
\end{tabular}
\caption{Effect of Technique II (image-splitting with size $\tilde{N})$ in Algorithm \ref{algo:surro-splitting}: visualization of boundary artifacts.}
\label{fig:dics-n}
\end{figure}

In this section, we again fix $\tau^{(t)} = 10^{-4}$. Convergence comparisons are shown in~\fig{compare_n}, and the dictionaries obtained with different $\tilde{N}$ are displayed \fig{dics-n}. In our experiments, we only consider square signals ($\tilde{N} = 12\times12, 16\times16, 32\times32, 64\times64, 256\times256$) and square dictionary kernels ($L = 12\times 12$).  When $\tilde{N}\geq2^2L$, say $\tilde{N}=32\times32$ or $\tilde{N}=64\times64$, the algorithm converges to a good functional value, which is the same as that without image-splitting. However, when $\tilde{N}$ is smaller than the threshold $2^2L$, say $\tilde{N}=16\times16$ or $12\times12$, the algorithm converges to a higher functional value, which implies worse dictionaries. Thus, we can conclude that \emph{the splitting size should be at least twice the dictionary kernel size in each dimension.  Otherwise, it will lead to boundary artifacts}. This phenomenon is consistent with the discussion in Section \ref{sec:region-sample}. The artifacts are specifically displayed in \fig{dics-n}. When $\tilde{N}$ is smaller than the threshold, say $12\times12$, the features learned are incomplete.

This section only studies the effect, due to boundary artifacts, of image-splitting on objective functional values. As Table \ref{tab:surro} shows, it also helps reducing computing time and memory cost, which is numerically validated in Section \ref{sec:acc}.

\subsubsection{Validation of Improvement III: stopping tolerance of FISTA $\tau^{(t)}$}
\label{sec:fix-vs-diminish}

\begin{figure}[t]
\centering
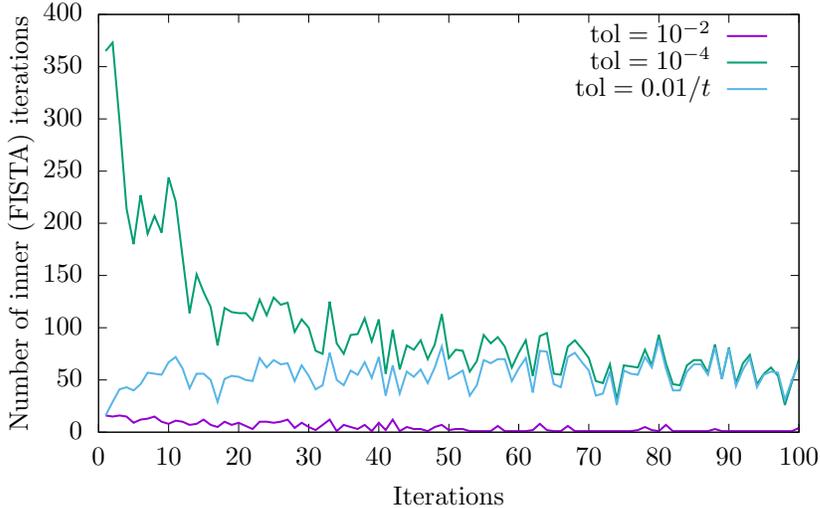
\caption{Effect of Technique III (stopping FISTA early) in Algorithm \ref{algo:surro-splitting}. Final objective with $\tau^{(t)}=10^{-4}$: $101.1$. Final objective with $\tau^{(t)}=10^{-2}$: $104.4$. Final objective with $\tau^{(t)}=0.01/t$: $101.3$, where $t$ is the iteration index with range $1\leq t \leq 100$. Our diminishing tolerance $\tau^{(t)} = 0.01/t$ provides a useful balance in that it reduces the number of FISTA iterations while losing little accuracy on the final functional value.
}
\label{fig:fix_vs_tol}
\end{figure}

In this section, we fix $p=10, \tilde{N} = 256\times256$ (no splitting).
\fig{fix_vs_tol} shows the effect of using different $\tau^{(t)}$. Using a small stopping tolerance $\tau^{(t)}=10^{-4}$ leads to a good functional value $101.1$ but large number of FISTA iterations, while a large tolerance $10^{-2}$ leads to a large functional value $104.4$ and small number of FISTA iterations. Consider our proposed diminishing tolerance rule (\ref{eqn:stop_condition}) $\tau^{(t)} = 0.01/t$. When the algorithm starts, $t=1$, we have $\tau^{(1)} = 10^{-2}$. At the end of the algorithm, $t=100$, $\tau^{(100)} = 10^{-4}$. Based on the results in \fig{fix_vs_tol}, our diminishing tolerance avoids large number of FISTA loops, especially at the initial steps, while losing little accuracy, as the final objective, $101.3$ is close to $101.1$.

\subsubsection{Validation of Improvement IV: computational techniques}
\label{sec:acc}

In this section, we fix $p=10,\tau^{(t)}=0.01/t$, and compare the calculation time and memory usage of spatial-domain update and frequency-domain update.  Table \ref{tab:spatial_vs_freq} illustrates that image-splitting helps reduce the single-step complexity and memory usage for both Option I (spatial-domain update) and Option II (frequency-domain update). For option II, the advantage of smaller splitting size $\tilde{N}$ is more significant than that of option I. When $\tilde{N}=256\times256$, option I is much better than option II; but when $\tilde{N}=64\times64$, the single step time of option II is comparable with that of option I. The reason for this is that, for option I, reducing $\tilde{N}$ only helps reduce the single-step time cost of CBPDN, updating Hessian matrix $A^{(t)}$ and the loops of FISTA, but does not help reduce the time cost of single-step time cost in FISTA. However, for option II, image-splitting not only reduces those three complexities, but also reduces the single-step complexity of FISTA. Furthermore, option II uses much less memory than option I when $\tilde{N}=64\times64$.

\begin{table}[]
\centering
\begin{tabular}{|c|c|c|c|c|r|}
\hline
\multirow{2}{*}{$\tilde{N}$} & \multicolumn{4}{c|}{Average single-step complexity (seconds)}                                                         & \multirow{2}{*}{\begin{tabular}[c]{@{}c@{}}Memory \\ Usage (MB)\end{tabular}} \\ \cline{2-5}
                   & CBPDN & Update $A^{(t)}$ & \begin{tabular}[c]{@{}c@{}}FISTA\\  (Loops $\times$ Single step)\end{tabular} & Total &                                  \\ \hline\hline
\multicolumn{6}{|c|}{Update in the spatial domain with dense matrix}     \\\hline
$256\times256$       & 14.8  & 25.1  & 57 $\times$ 0.017                            & 40.9 &                     3058.56           \\ \hline
$128\times128$       & 3.42  & 6.80  & 37 $\times$ 0.017                            & 10.8  &                       1258.37           \\ \hline
$64\times64$        & 1.05  & 2.25  & 24 $\times$ 0.017                            & 3.71 &                      808.32            \\ \hline\hline
\multicolumn{6}{|c|}{(Option I) Update in the spatial domain with sparse matrix}     \\\hline
$256\times256$       & 14.8  & 4.47  & 57 $\times$ 0.017                            & 20.3  &                     486.91             \\ \hline
$128\times128$       & 3.42  & 1.77  & 37 $\times$ 0.017                            & 5.82  &                       366.51           \\ \hline
$64\times64$        & 1.05  & 0.84  & 24 $\times$ 0.017                            & 2.30  &                      342.90            \\ \hline\hline
\multicolumn{6}{|c|}{(Option II) Update in the frequency domain (including extra time caused by FFT)}   \\\hline
$256\times256$       & 14.8  & 0.89  & 57 $\times$ 1.068                            & 76.6 &                2458.84                  \\ \hline
$128\times128$       & 3.42  & 0.22  & 37 $\times$ 0.244                            & 12.7 &                 622.28                 \\ \hline
$64\times64$        & 1.05  & 0.06  & 24 $\times$ 0.072                            & 2.84  &                    158.11              \\ \hline
\end{tabular}
\label{tab:spatial_vs_freq}
\caption{Comparison of two options in Algorithm \ref{algo:surro-splitting} with different splitting size $\tilde{N}$. $\lambda=0.1$, average density of $X$: 0.0037. This is the validation of Table \ref{tab:surro}. When $\tilde{N}=64\times64$, the two options share similar performance: Option I is better on time cost and Option II is better on memory cost. Image-splitting is necessary for Option II, but not necessary for Option I.}
\end{table}

\fig{sparse_n} and \fig{freq_n} compare the objective functional value versus time. \fig{sparse_n} indicates that reducing $\tilde{N}$ does \emph{not} help a lot for Option I. Table \ref{tab:spatial_vs_freq} shows that smaller $\tilde{N}$ reduces the single step complexity, but it also reduces the gain in each step because a smaller splitting size leads to less information used for training. This is a trade-off. By \fig{sparse_n}, $\tilde{N}=128\times128$ is a good choice.

Option II, in contrast, benefits more from smaller $\tilde{N}$, as can be seen from \fig{freq_n} and Table \ref{tab:spatial_vs_freq}. Although splitting a training image reduces the gain in each step, the benefit overwhelms the loss. Thus, for Option II, the smaller the splitting size the better, as long as $\tilde{N}$ is larger than the threshold for boundary artifacts.

\begin{figure}[t]
\centering \small
\begin{tabular}{cc}
\subfigure[][\parbox{5.6cm}{Algorithm \ref{algo:surro-splitting} Option I.  $\tilde{N}=128\!\times\!128$ is a good choice.}]{
       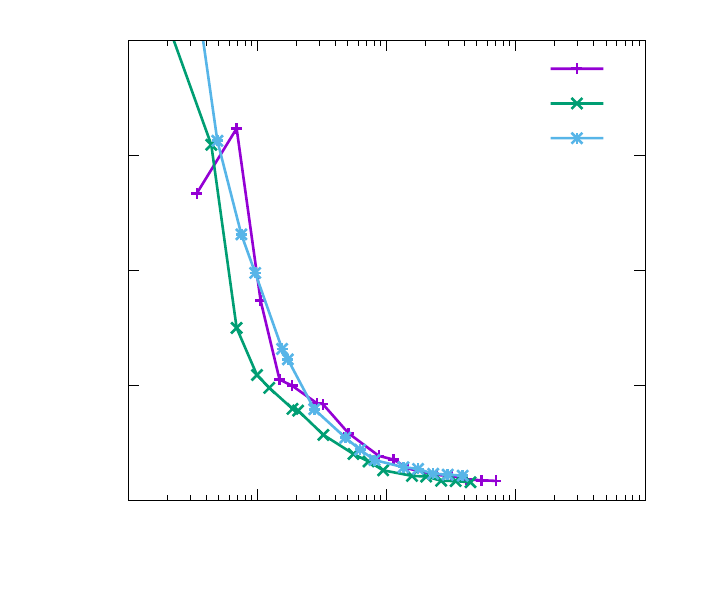
\label{fig:sparse_n}}
&
\hspace{-11mm} \subfigure[][\parbox{5.6cm}{Algorithm \ref{algo:surro-splitting} Option II. $\tilde{N}=64\times64$ converges fast.}]{
       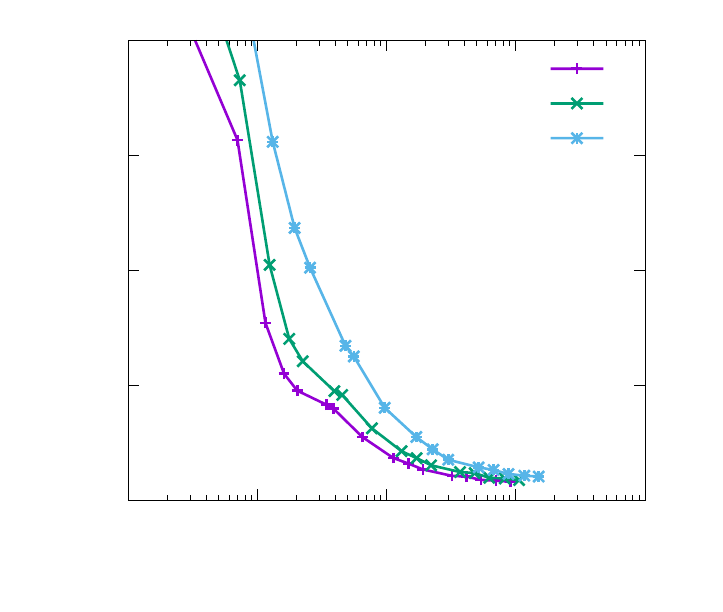
\label{fig:freq_n}}
\end{tabular}
\caption{Effect of splitting region size $\tilde{N}$ on different options.}
\label{fig:suro_n_speed}
\end{figure}

\subsection{Main result I: convergence speed}
\label{sec:compare_methods}

\begin{figure}[t]
\centering \small
\begin{tabular}{cc}
\subfigure[][\parbox{5.6cm}{Methods with spatial-domain dictionary update scheme. Online algorithms, both Algorithm \ref{algo:sgd} and \ref{algo:surro-splitting}, outperform batch method (Papyan et al. \cite{papyan2017convolutional}), Algorithm \ref{algo:surro-splitting} performs the best.}]{
  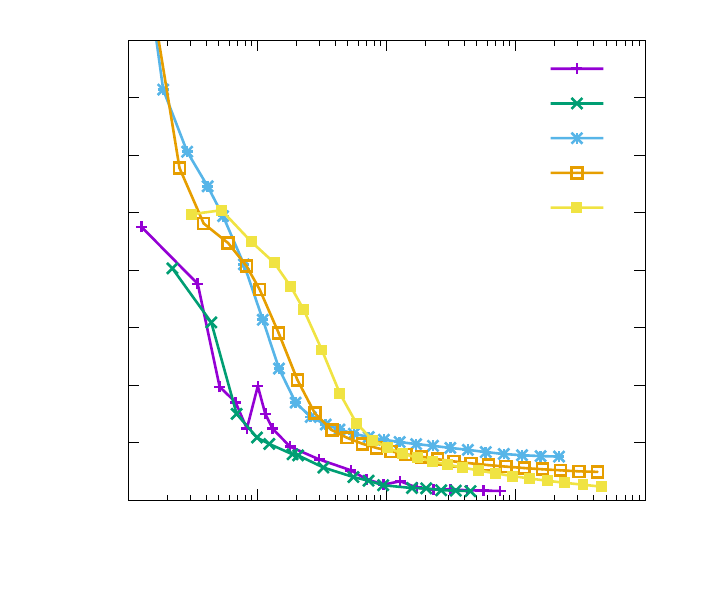
\label{fig:methods_spatial}}
&
\hspace{-11mm} \subfigure[][\parbox{5.6cm}{Methods with frequency-domain dictionary update scheme.  In this plot, ``Prev. Online'' refers to our algorithm ``Online-Samp'' proposed in~\cite{liu-2017-online}. Both the online algorithms converge faster than batch methods (ADMM consensus dictionary update~\cite{sorel-2016-fast, garcia-2017-subproblem}).}
]{
 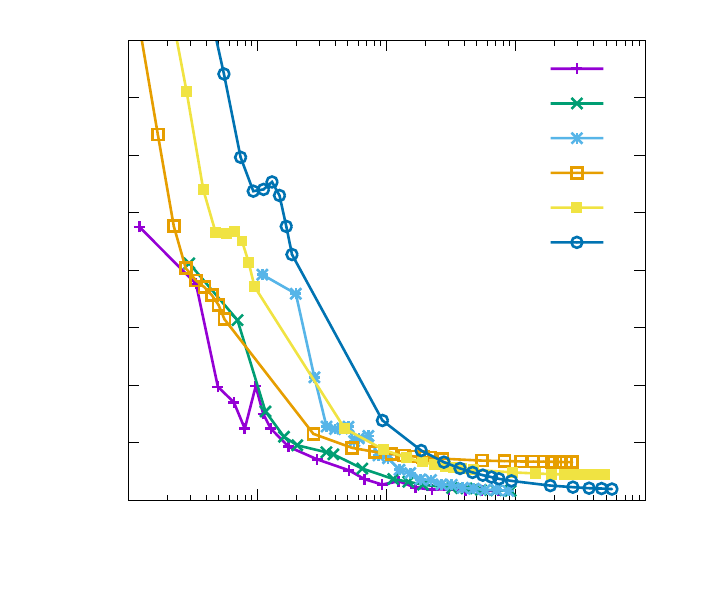
\label{fig:methods_freq}}
\end{tabular}
\caption{Main Result I: convergence speed comparison between online algorithms and batch algorithms.}
\label{fig:compare_methods}
\end{figure}

\begin{figure}[t]
\centering \small
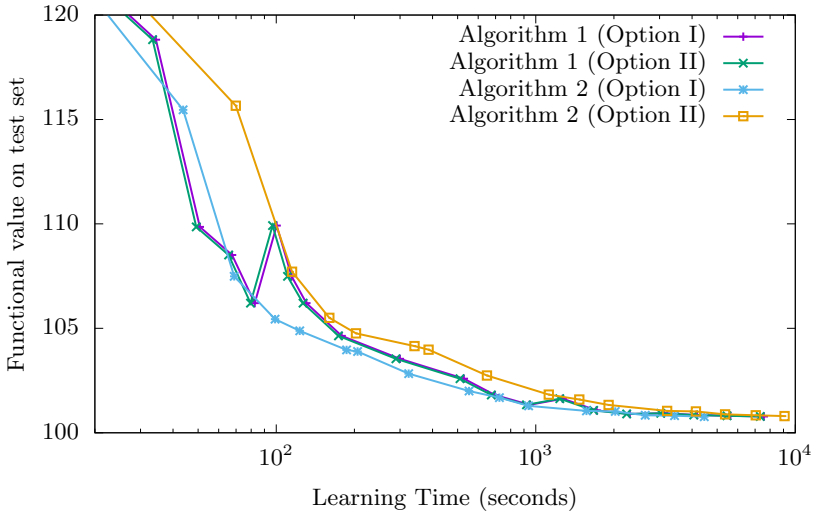
\caption{Main Result I: convergence speed comparison between online methods in this article.}
\label{fig:methods_online}
\end{figure}

In this section, we study the convergence speeds of all the methods on the clean data set, without a masking operator. We compare our methods with two leading batch learning algorithms: the method of Papyan et al. \cite{papyan2017convolutional}, which uses $K$-SVD and updates the dictionary in the spatial domain, and an algorithm~\cite{garcia-2017-subproblem} that uses the ADMM consensus dictionary update~\cite{sorel-2016-fast}, which is computed in the frequency domain. For batch learning algorithms, we test on subsets of 10, 20, and 40 images selected from the training set. For online learning algorithms, since they are scalable in the size of the training set, we just test our methods on the whole training set of 40 images.  All the parameters are tuned as follows.
For batch learning algorithm (Papyan et al.), we use the software they released, and for batch learning algorithm (ADMM consensus update), we use the ``adaptive penalty parameter'' scheme in \cite{wohlberg2017admm}.
For modified SGD (Algorithm \ref{algo:sgd}), we use the step size of $10/(5+t)$. For Surrogate-Splitting (Algorithm \ref{algo:surro-splitting}), we use $p=10, \tau^{(t)}=0.01/t, \tilde{N}=128\times128 $ for spatial-domain update, $\tilde{N}=64\times64$ for frequency-domain update, as we tuned in the previous sections.  For our algorithm proposed in~\cite{liu-2017-online}, we use $p=10, \tau^{(t)}=10^{-3}, \tilde{N}=64\times64$.

The performance comparison of batch and online methods is presented in~\fig{compare_methods}. The advantage of online learning is significant (note
 that the time axis is logarithmically scaled). To obtain the same functional value $101$ on the test set, batch learning takes 15 hours, our previous method~\cite{liu-2017-online} takes around 1.5 hours, Algorithm \ref{algo:surro-splitting} with option II takes around 1 hour, and Algorithm \ref{algo:sgd} and Algorithm \ref{algo:surro-splitting} with option I takes less than 1 hour. We can conclude that, both modified SGD (Algorithm \ref{algo:sgd}) and Surrogate-Splitting (Algorithm \ref{algo:surro-splitting}) converge faster than the batch learning algorithms and our previous method.

\subsection{Main result II: memory usage}

As Table \ref{tab:memory} shows, both Algorithm \ref{algo:sgd} and \ref{algo:surro-splitting} save a large amount of memory.
\begin{table}[t]
\centering
\begin{tabular}{|c|r|}
\hline
Scheme & Memory (MB) \\
\hline \hline
Batch learning (consensus update, batch $K=10$) & 1959.58\\
\hline
Batch learning (consensus update, batch $K=20$)& 3887.08\\
\hline
Batch learning  (consensus update, batch $K=40$)& 7742.08 \\
\hline
Batch learning  (Papyan et al. \cite{papyan2017convolutional}, batch $K=10$) & 1802.29\\
\hline
Batch learning  (Papyan et al. \cite{papyan2017convolutional}, batch $K=20$)& 3390.24\\
\hline
Batch learning  (Papyan et al. \cite{papyan2017convolutional}, batch $K=40$)& 6566.15 \\
\hline
Our algorithm ``Online-Samp'' in~\cite{liu-2017-online} & 158.11 \\
\hline
Algorithm \ref{algo:sgd} Option I (sgd-spatial) & 111.38 \\
\hline
Algorithm \ref{algo:sgd} Option II (sgd-frequency) & 154.84 \\
\hline
Algorithm \ref{algo:surro-splitting} Option I (surro-spatial) & 342.90 \\
\hline
Algorithm \ref{algo:surro-splitting} Option II (surro-frequency) & 158.11 \\
\hline
\end{tabular}
\caption{Main Result II: Memory Usage Comparison in Megabytes.
}
\label{tab:memory}
\end{table}

\subsection{Main result III: dictionaries obtained by different algorithms}

In \fig{dics} we display the dictionaries obtained by the algorithms in Section \ref{sec:compare_methods}. A small training set, say 10 images, leads to some random kernels in the dictionaries. A training set containing 40 images works much better. Our algorithms can learn comparable dictionaries (see \fig{dic40}, \fign{dic-surro}, and \fign{dic-sgd}) within much less time (see \fig{compare_methods}) and much less memory usage (see \tbl{memory}).
\begin{figure}
\centering
\begin{tabular}{ccc}
\hspace{-12mm}
\subfigure[][\parbox{0.25\textwidth}{Dictionaries learned by batch learning algorithm (consensus update, 10 training images): many ``random" kernels.}]{
\includegraphics[width=0.42\textwidth]{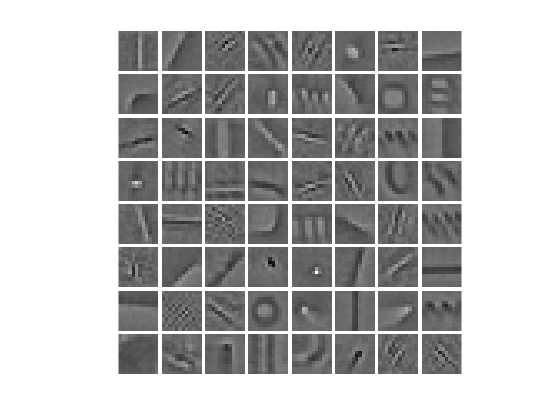}\label{fig:dic10}}
&
\hspace{-15mm}
\subfigure[][\parbox{0.25\textwidth}{Dictionaries learned by batch learning algorithm (consensus update, 20 training images): less ``random" kernels, more valid features.}]{
\includegraphics[width=0.42\textwidth]{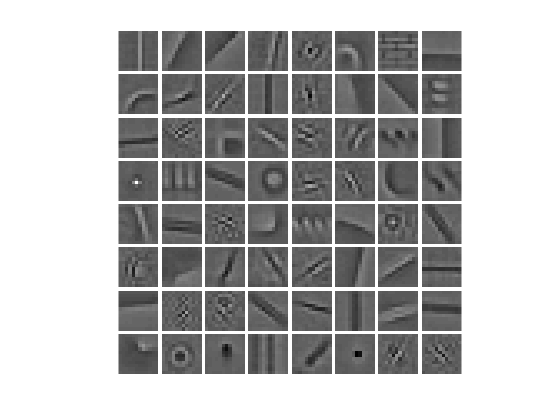}\label{fig:dic20}}
&
\hspace{-15mm}
\subfigure[][\parbox{0.25\textwidth}{Dictionaries learned by batch learning algorithm (consensus update, 40 training images): almost all kernels are valid.}]{
\includegraphics[width=0.42\textwidth]{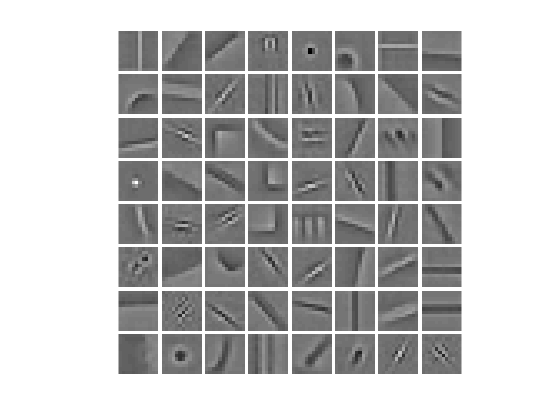}\label{fig:dic40}}
\\
\hspace{-12mm}
\subfigure[][\parbox{0.25\textwidth}{Dictionaries learned by batch learning algorithm (Papyan et al. \cite{papyan2017convolutional}, 10 training images): many ``random" kernels.}]{
\includegraphics[width=0.42\textwidth]{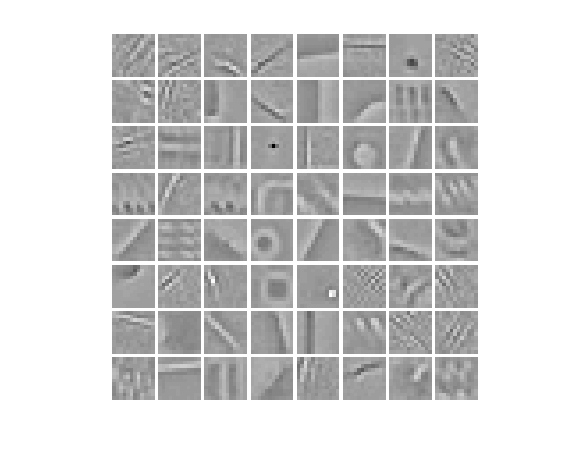}\label{fig:papyan10}}
&
\hspace{-15mm}
\subfigure[][\parbox{0.25\textwidth}{Dictionaries learned by batch learning algorithm (Papyan et al. \cite{papyan2017convolutional}, 20 training images): less ``random" kernels, more valid features.}]{
\includegraphics[width=0.42\textwidth]{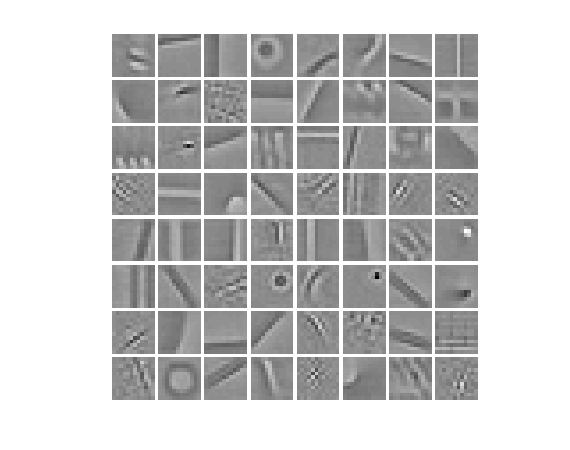}\label{fig:papyan20}}
&
\hspace{-15mm}
\subfigure[][\parbox{0.25\textwidth}{Dictionaries learned by batch learning algorithm (Papyan et al. \cite{papyan2017convolutional}, 40 training images): almost all kernels are valid.}]{
\includegraphics[width=0.42\textwidth]{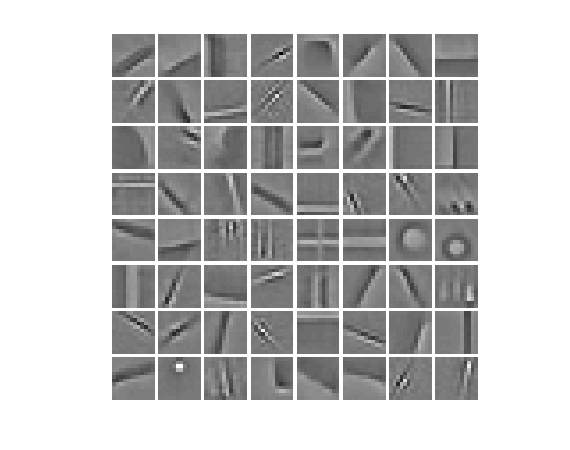}\label{fig:papyan40}}
\\
\hspace{-12mm}
\subfigure[][\parbox{0.22\textwidth}{Dictionaries learned by \cite{liu-2017-online}, almost all kernels are valid.}]{
\includegraphics[width=0.42\textwidth]{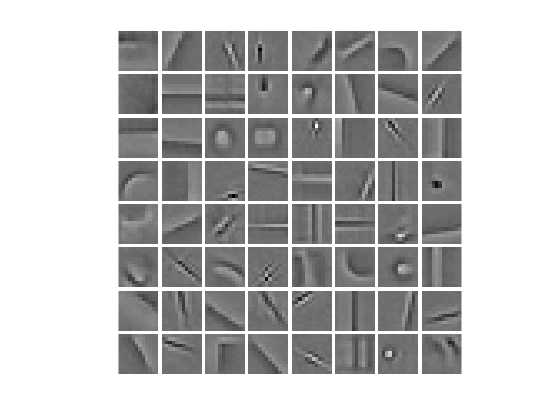}\label{fig:dic-icip}}
&
\hspace{-15mm}
\subfigure[][\parbox{0.22\textwidth}{Dictionaries learned by Algorithm \ref{algo:sgd}, almost all kernels are valid.}]{
\includegraphics[width=0.42\textwidth]{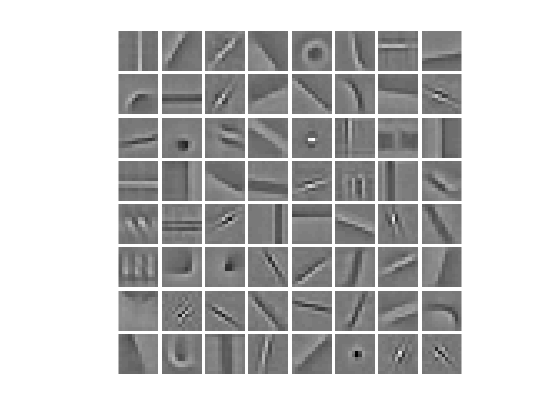}\label{fig:dic-sgd}}
&
\hspace{-15mm}
\subfigure[][\parbox{0.22\textwidth}{Dictionaries learned by Algorithm \ref{algo:surro-splitting}, almost all kernels are valid.}]{
\includegraphics[width=0.42\textwidth]{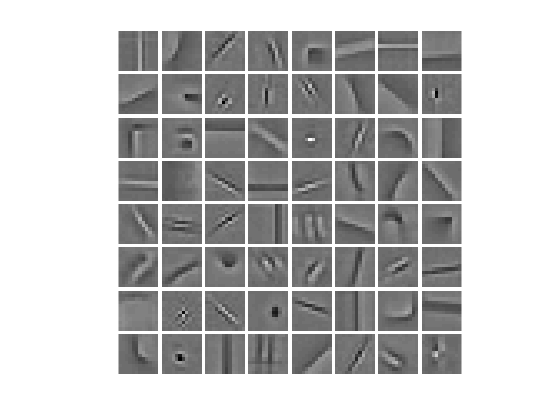}\label{fig:dic-surro}}
\end{tabular}
\caption{Main result III: A comparison of dictionaries learned using different algorithms.
}
\label{fig:dics}
\end{figure}

\section{Numerical results: learning from the noisy data set}\label{sec:msk_reslt}

\begin{figure}
\centering
\begin{tabular}{ccc}
\hspace{-12mm}
\subfigure[][\parbox{0.28\textwidth}{One of the training images. (10\% pixels corrupted (salt-and-pepper noise))}]{
\includegraphics[width=0.42\textwidth]{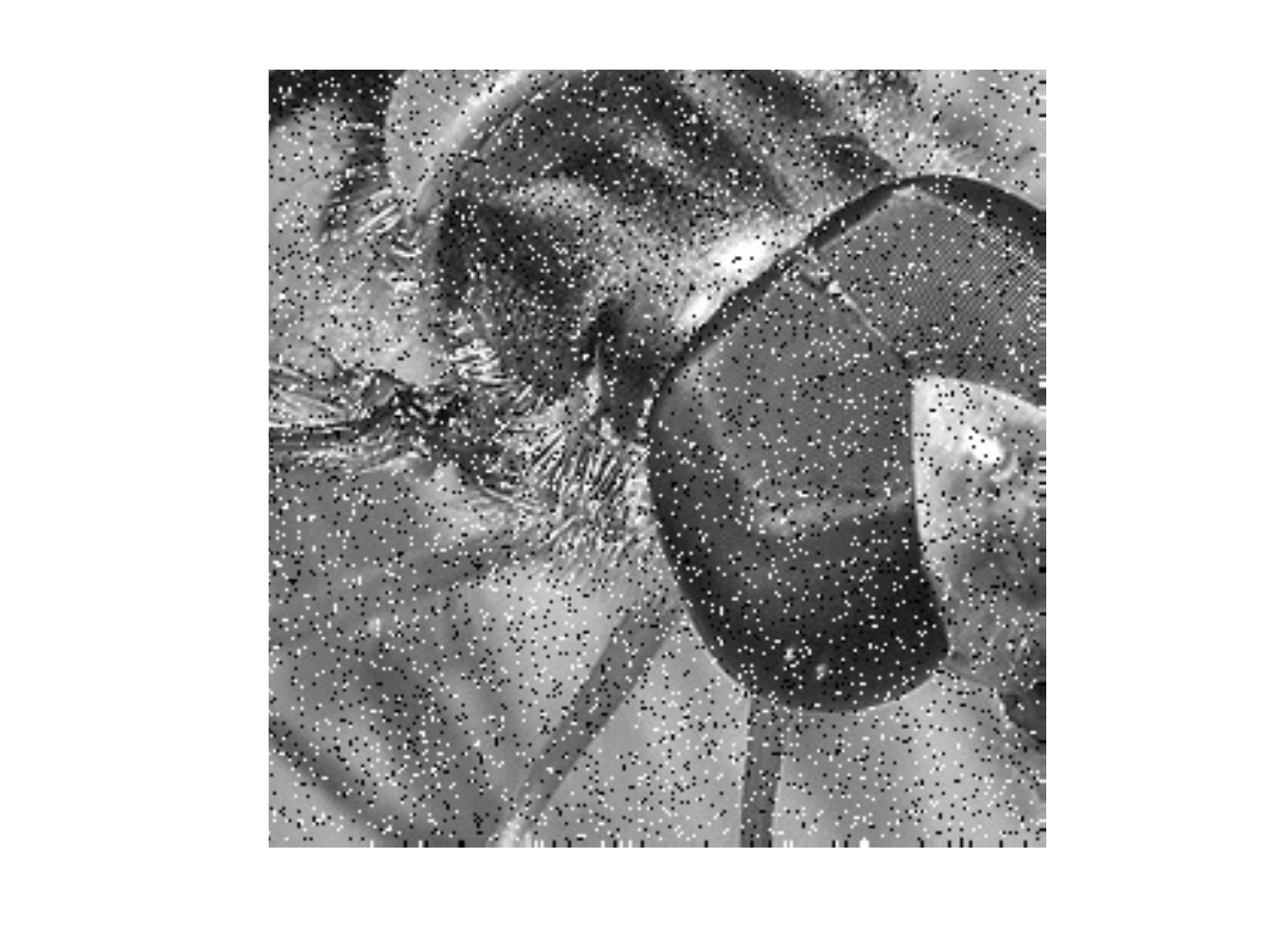}\label{fig:pic01}}
&
\hspace{-15mm}
\subfigure[][\parbox{0.22\textwidth}{Learning with normal Algorithm \ref{algo:sgd}: some features learned.}]{
\includegraphics[width=0.42\textwidth]{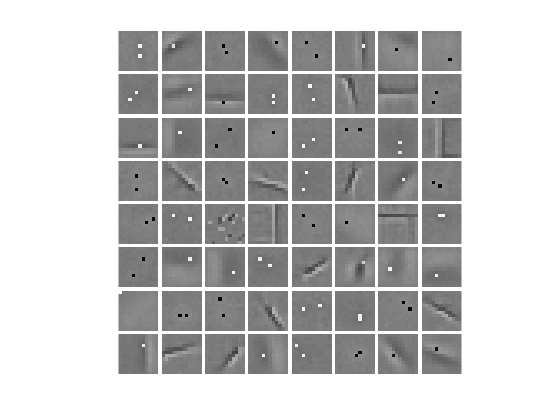}\label{fig:dic01}}
&
\hspace{-15mm}
\subfigure[][\parbox{0.28\textwidth}{Learning with Algorithm \ref{algo:sgd} on masked loss function: clean features learned.}]{
 \includegraphics[width=0.42\textwidth]{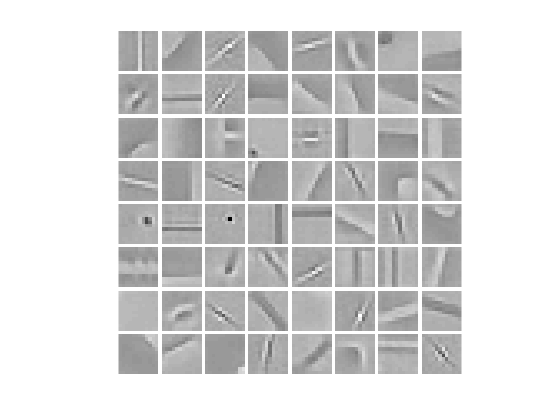}\label{fig:dicms01}}
\\
\hspace{-12mm}
\subfigure[][\parbox{0.28\textwidth}{One of the training images. (20\% pixels corrupted (salt-and-pepper noise))}]{
 \includegraphics[width=0.42\textwidth]{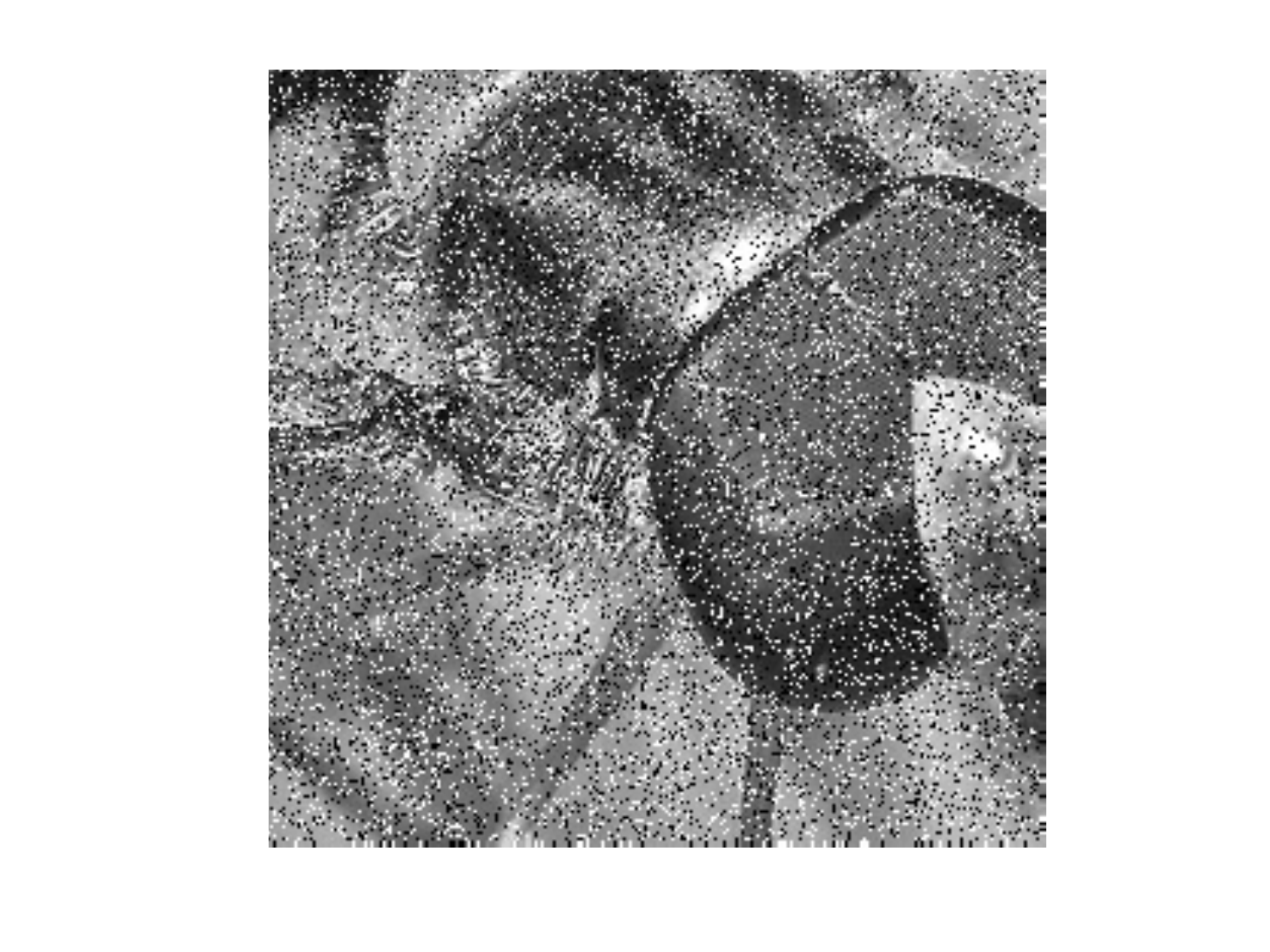}\label{fig:pic05}}
&
\hspace{-15mm}
\subfigure[][\parbox{0.22\textwidth}{Learning with normal Algorithm \ref{algo:sgd}: few features learned.}]{
\includegraphics[width=0.42\textwidth]{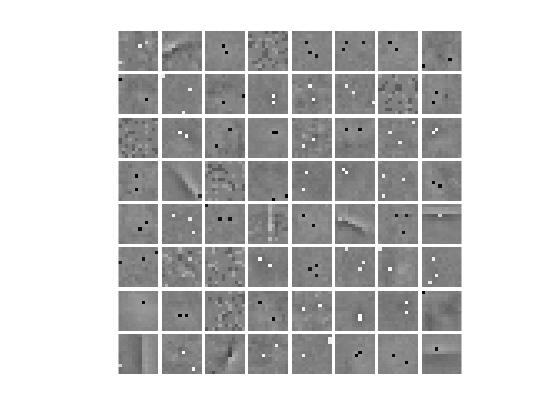}\label{fig:dic05}}
&
\hspace{-15mm}
\subfigure[][\parbox{0.28\textwidth}{Learning with Algorithm \ref{algo:sgd} on masked loss function: clean features learned.}]{
\includegraphics[width=0.42\textwidth]{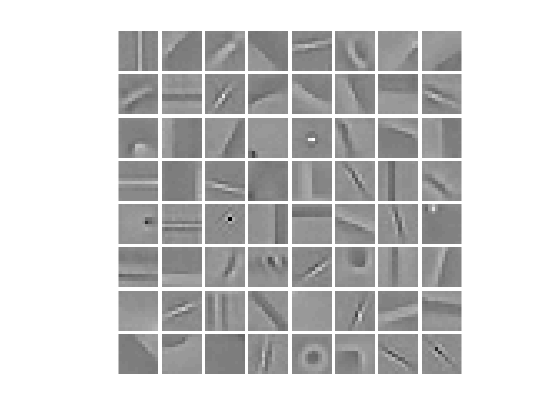}\label{fig:dicms05}}
\\
\hspace{-12mm}
\subfigure[][\parbox{0.28\textwidth}{One of the training images. (30\% pixels corrupted (salt-and-pepper noise))}]{
\includegraphics[width=0.42\textwidth]{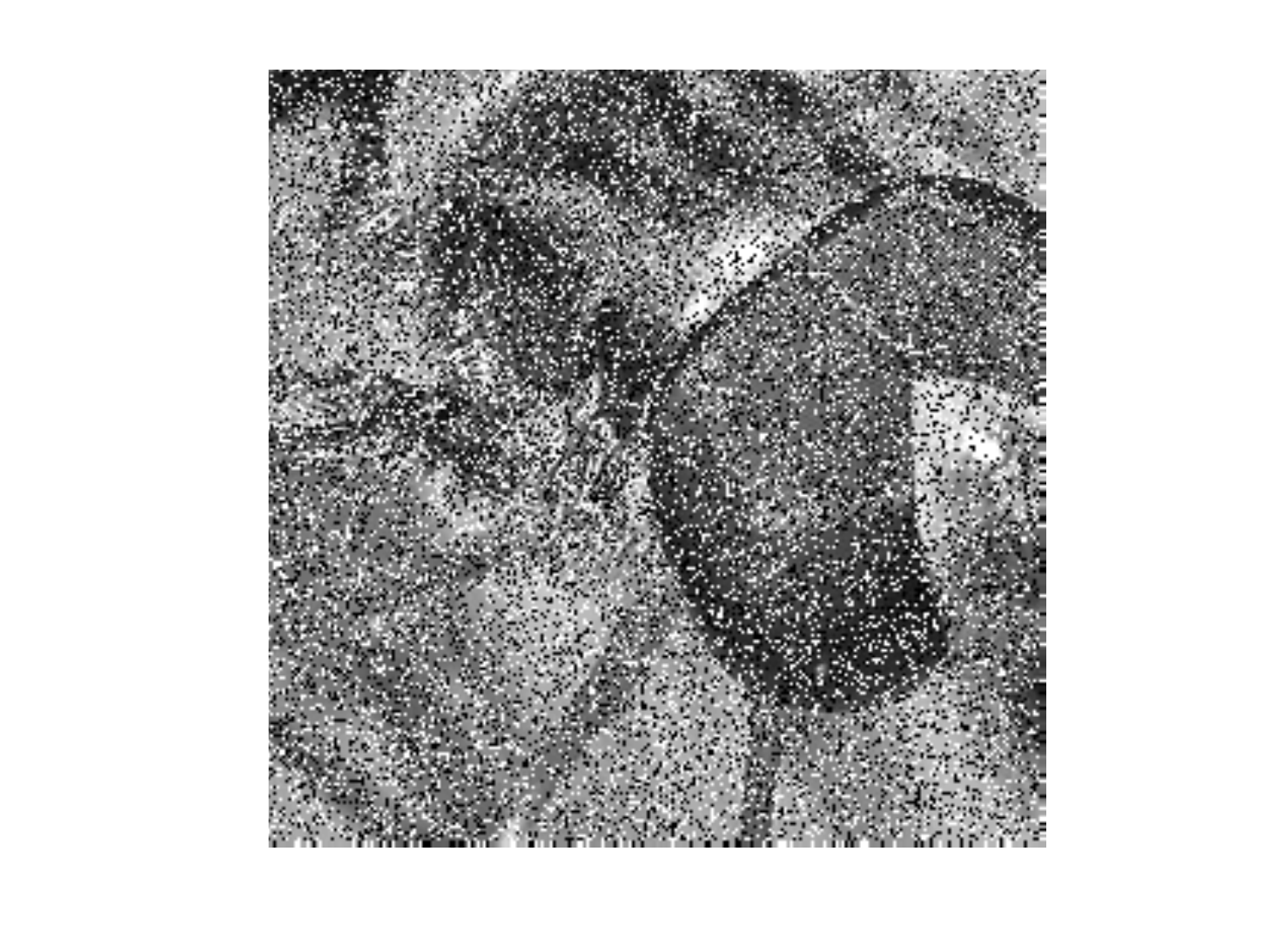}\label{fig:pic1}}
&
\hspace{-15mm}
\subfigure[][\parbox{0.22\textwidth}{Learning with normal Algorithm \ref{algo:sgd}: almost no valid features.}]{
 \includegraphics[width=0.42\textwidth]{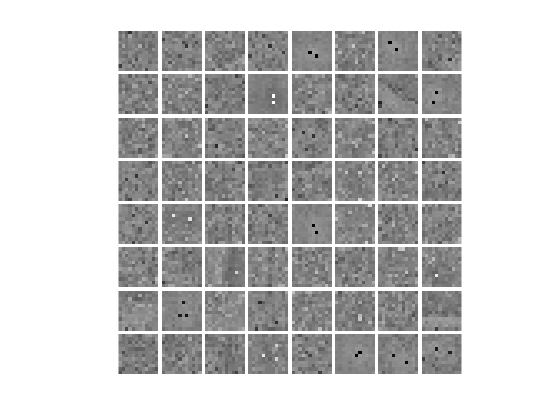}\label{fig:dic1}}
&
\hspace{-15mm}
\subfigure[][\parbox{0.28\textwidth}{Learning with Algorithm \ref{algo:sgd} on masked loss function: clean features learned.}]{
 \includegraphics[width=0.42\textwidth]{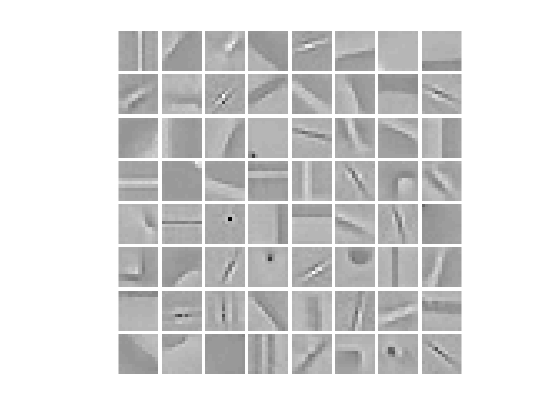}\label{fig:dicms1}}
\end{tabular}
\caption{Learning from the noisy training set.}
\label{fig:dics-ms}
\end{figure}

In this section, we try to learn dictionaries from noisy images. We test the algorithms on the training set with \emph{salt-and-pepper} impulse noise at known pixel locations.  We apply the noise to $10\%, 20\%$, and $30\%$ of the pixels, as \fig{dics-ms} shows.  We use the data set with 40 training images and 20 testing images with uniform size $256\times256$, which are the same with those in Section \ref{sec:rslt}. All the images are pre-processed by applying a highpass filter computed as the difference between the input and a non-linear lowpass filter\footnote{The lowpass component was computed by $\ell_2$ total-variation denoising~\cite{rudin-1992-nonlinear} of the input image with a spatial mask informed by the known locations of corrupted pixels in the data fidelity term.}.  When the number of noisy pixels is low, say $10\%$, SGD without masking (Algorithm \ref{algo:sgd}) still can learn some features, as \fig{dic01} demonstrates. When the number of noisy pixels is significant, say $30\%$, SGD without masking ``learns'' nothing valid, as \fig{dic1} demonstrates. However, SGD with masking technique works much better because it ``ignores'' the noisy pixels.

\subsection{Masked CDL: online algorithms vs batch algorithm}

We compare our algorithms with masked loss function and a batch dictionary learning algorithm\footnote{We used the implementation \texttt{cbpdndlmd.m} from the Matlab version of the SPORCO library~\cite{wohlberg-2016-sporco}, with the Iterated Sherman-Morrison dictionary update solver option~\cite{wohlberg-2016-efficient}.}~\cite{wohlberg-2016-boundary} incorporating the mask via the mask decoupling technique~\cite{heide-2015-fast}.  We use Additive Mask Simulation (AMS)~\cite{wohlberg-2016-boundary} to solve sparse coding step with the masked objective function.  The parameters for Algorithms \ref{algo:sgd} and \ref{algo:surro-splitting} are chosen similarly as those in Section \ref{sec:rslt}. For Algorithm \ref{algo:sgd}, we choose $\eta^{(t)} = 10/(5+t)$ and Option I. For Algorithm \ref{algo:surro-splitting}, we choose $p=10,\tilde{N}=128^2,\tau^{(t)} = 0.01/(10+t)$ and Option I.  A comparison of functional values on the \emph{noise-free} test set is shown in \fig{methods_ms}.  Our online algorithms converge much faster and more stably. Algorithms \ref{algo:sgd} and \ref{algo:surro-splitting} take around $1$ hour to converge, while the mask-decoupling scheme requires more than $10$ hours.

\begin{figure}[t]
\centering \small
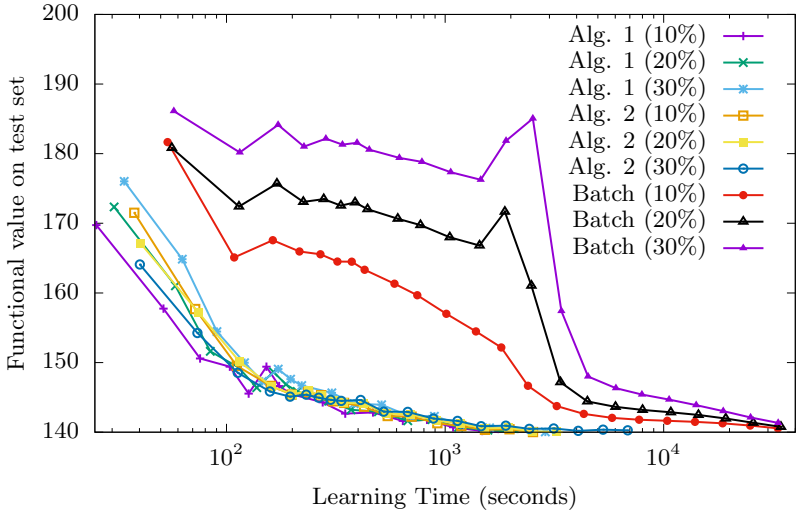
\caption{Learning from the noisy training set: speed test. Both Algorithms \ref{algo:sgd} and \ref{algo:surro-splitting} use Option I. ``Batch'' refers to Iterated Sherman-Morrison dictionary update~\cite{wohlberg-2016-efficient} with mask decoupling technique~\cite{heide-2015-fast}, as in~\cite{wohlberg-2016-boundary}.}
\label{fig:methods_ms}
\end{figure}

\section{Numerical results: learning from large data set}

In this section, we test the feasibility of our methods on large data set, which is not tractable for batch methods. This training set consists of 1000 images of size $256\times256$ selected from the MIRFLICKR-1M dataset, and the testing set consists of 50 distinct images with the same size from the same source. A dictionary of 100 kernels with size $12\times12$ is trained and the related experiment results are reported in \fig{large}.

The parameters for Algorithms \ref{algo:sgd} and \ref{algo:surro-splitting} are chosen the same as those in Section \ref{sec:rslt}. For Algorithm \ref{algo:sgd}, we choose $\eta^{(t)} = 10/(5+t)$ and Option I. For Algorithm \ref{algo:surro-splitting}, we choose $p=10,\tilde{N}=128^2,\tau^{(t)} = 0.01/t$ and Option I.

\begin{figure}[htbp]
\centering \small
\begin{tabular}{cc}
\subfigure[Time cost on large data set.]{
       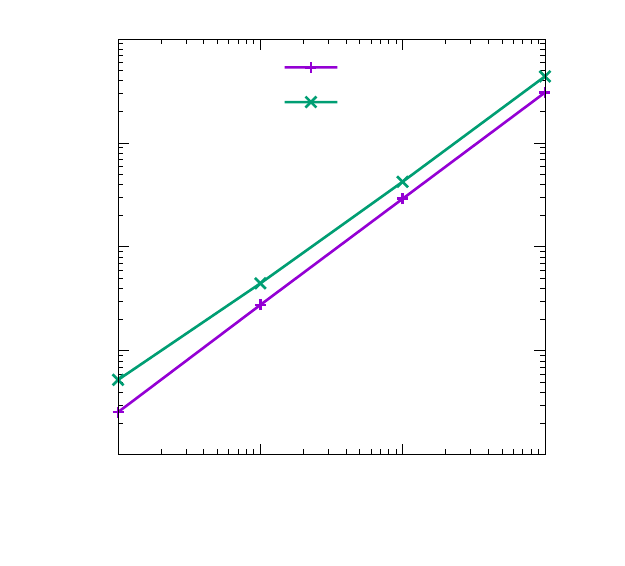
\label{fig:large_t}}
&
\hspace{-9mm} \subfigure[Memory cost on large data set.]{
      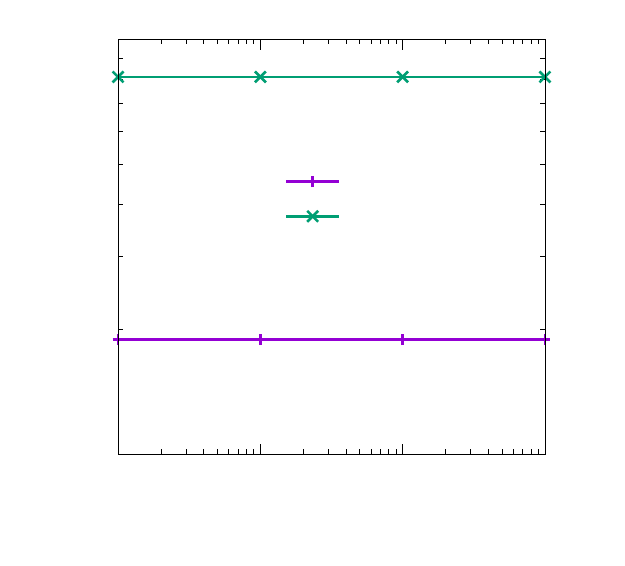
\label{fig:large_m}}
\\
 \subfigure[Functional value on large data set.]{
       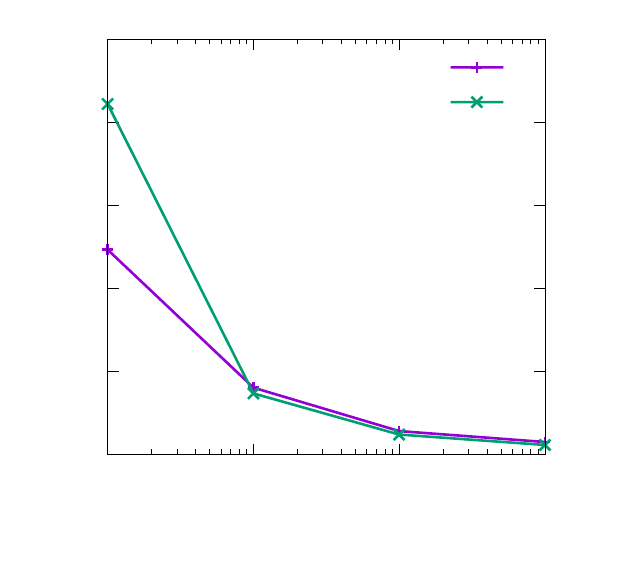
\label{fig:large_f}}
&
\hspace{-9mm} \subfigure[Functional value with training time.]{
       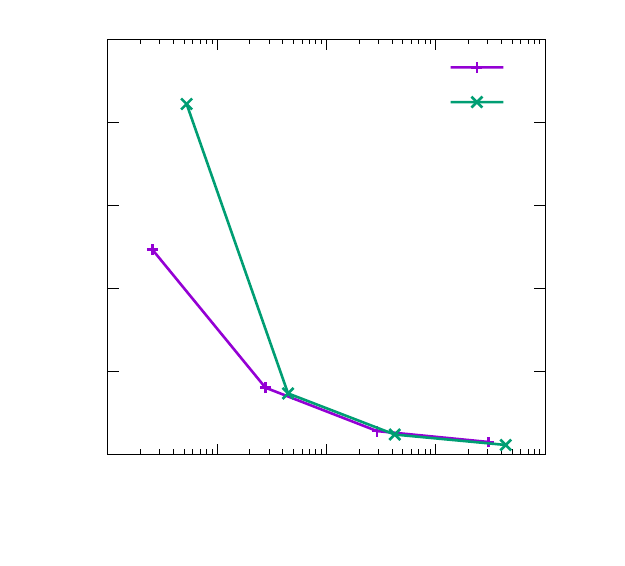
\label{fig:large_ft}}
\end{tabular}
\caption{Learning from a large data set.}
\label{fig:large}
\end{figure}

Unlike the experiments in Section \ref{sec:rslt}, we run our algorithms with only one epoch, \ie true online learning. Results in \fig{large} demonstrate that our Algorithm \ref{algo:sgd} and \ref{algo:surro-splitting} are both feasible on large data set. The first order method, Algorithm \ref{algo:sgd}, has a cheaper single step, but it learns less with the same number of iterations. The second order method, Algorithm \ref{algo:surro-splitting}, has the converse behavior, achieveing a slightly smaller functional value with the same number of iterations. Finally,~\fig{large_ft} shows that the two algorithms have similar performance in terms of functional value on the test set with respect to training time. This result is consistent with those in Section \ref{sec:rslt}.

\section{Conclusions}
\label{sec:concl}

We have proposed two efficient online convolutional dictionary learning methods. Both of them have a theoretical convergence guarantee and show good performance on both time and memory usage. Compared to recent online CDL works \cite{degraux-2017-online, wang2017online}, which use the same framework but different $D$-update algorithms, our second-order method improves the framework by several practical techniques. Our first-order method, to the best of our knowledge, is the first attempt to use first order methods in online CDL.  It shows better performance in time and memory usage, and requires fewer parameters to tune.  Moreover, based on these two methods, we have also proposed an online dictionary learning method, which is able to learn meaningful dictionaries from a partially masked training set. Although only single-channel images are considered in this article, our online methods can easily be extended to the multi-channel case~\cite{wohlberg-2016-convolutional}.

\section*{Acknowledgement} The authors thank the two anonymous reviewers for their careful reading and valuable comments that helped improve the final version of this manuscript.

\appendix

\section{Derivation of conjugate cogradient (\ref{eqn:basic_gradient_freq}) of the frequency-domain loss function $\hat{l}$}
\label{app:derivation}

Consider a real-valued function defined on the complex domain $f:\C^n \to \R$, which can be viewed as a function defined on the $2n$ dimensional real domain: $f(x) = f\big(\Re(x) + i \Im(x)\big)$, where $\Re(x), \Im(x) \in \R^n$ are the real part and imaginary part, respectively. By \cite{sorber2012unconstrained}, ``conjugate cogradient'' is defined as
\begin{equation}
\label{eqn:cogra}
\nabla f(x) \triangleq \frac{\partial f}{\partial \Re(x)} + i \frac{\partial f}{\partial \Im(x)} \;.
\end{equation}
Based on (\ref{eqn:cogra}), we give a derivation of (\ref{eqn:basic_gradient_freq}).

Recall the definition $\hat{l}(\hat{\mb{d}},\hat{\mb{x}};\hat{\mb{s}}) = 1/2\normsz[\big]{\hat{X}\hat{\mb{d}}-\hat{\mb{s}}}^2$. Substituting $\hat{X} = \Re(\hat{X}) + i \Im(\hat{X})$, $\hat{\mb{d}}=\Re(\hat{\mb{d}})+i \Im(\hat{\mb{d}})$, and $\hat{\mb{s}} = \Re(\hat{\mb{s}}) + i \Im(\hat{\mb{s}})$ into $\hat{l}$, we have
\[
\begin{aligned}
&\hat{l}(\hat{\mb{d}},\hat{\mb{x}};\hat{\mb{s}})  \\
=& \frac{1}{2}\normsz[\big]{\Re(\hat{X})\Re(\hat{\mb{d}}) - \Im(\hat{X})\Im(\hat{\mb{d}})-\Re(\hat{\mb{s}}) + i\big(\Im(\hat{X})\Re(\hat{\mb{d}})+ \Re(\hat{X}) \Im(\hat{\mb{d}}) - \Im(\hat{\mb{s}})\big)}^2\\
=&\frac{1}{2}\normsz[\big]{\Re(\hat{X})\Re(\hat{\mb{d}}) - \Im(\hat{X})\Im(\hat{\mb{d}}) - \Re(\hat{\mb{s}})}^2 + \frac{1}{2} \normsz[\big]{\Im(\hat{X})\Re(\hat{\mb{d}}) + \Re(\hat{X}) \Im(\hat{\mb{d}}) - \Im(\hat{\mb{s}})}^2 \;.
\end{aligned}
\]
The partial derivatives on $\Re(\hat{\mb{d}})$ and $\Im(\hat{\mb{d}})$ are, respectively,
\[
\begin{aligned}
\frac{\partial \hat{l}}{\partial \Re(\hat{\mb{d}})} =&  \Re(\hat{X})^T \big(\Re(\hat{X})\Re(\hat{\mb{d}}) - \Im(\hat{X})\Im(\hat{\mb{d}})-\Re(\hat{\mb{s}})\big) + \Im(\hat{X})^T \big(\Im(\hat{X})\Re(\hat{\mb{d}}) + \Re(\hat{X}) \Im(\hat{\mb{d}}) - \Im(\hat{\mb{s}})\big)\\
\frac{\partial \hat{l}}{\partial \Im(\hat{\mb{d}})} = & \Im(\hat{X})^T \big(-\Re(\hat{X})\Re(\hat{\mb{d}}) + \Im(\hat{X})\Im(\hat{\mb{d}}) + \Re(\hat{\mb{s}})\big)  + \Re(\hat{X})^T \big(\Im(\hat{X})\Re(\hat{\mb{d}}) + \Re(\hat{X}) \Im(\hat{\mb{d}}) - \Im(\hat{\mb{s}})\big) \;.
\end{aligned}
\]
Therefore,
\[
\begin{aligned}
&\hat{X}^H(\hat{X}\hat{\mb{d}}-\hat{\mb{s}})\\ = & (\Re(\hat{X}) - i \Im(\hat{X}))^T \Big( (\Re(\hat{X})\Re(\hat{\mb{d}}) - \Im(\hat{X})\Im(\hat{\mb{d}})-\Re(\hat{\mb{s}}))+ i \big( \Im(\hat{X})\Re(\hat{\mb{d}}) + \Re(\hat{X}) \Im(\hat{\mb{d}}) - \Im(\hat{\mb{s}}) \big) \Big)\\ = & \frac{\partial \hat{l}}{\partial R\hat{\mb{d}}} + i \frac{\partial \hat{l}}{\partial I\hat{\mb{d}}} \;.
\end{aligned}
\]
By the definition of conjugate cogradient (\ref{eqn:cogra}), the right side of the above equation is the conjugate cogradient of $\hat{l}$, i.e.
\[
\nabla \hat{l}(\hat{\mb{d}},\hat{\mb{x}};\hat{\mb{s}}) = \hat{X}^H(\hat{X}\hat{\mb{d}}-\hat{\mb{s}}) \;.
\]

\section{Proof of the equivalence between the gradients in the frequency domain and spatial domain: (\ref{eqn:equal_freq_spatial}) }
\label{app:equal}

\begin{proof}
Let $\F$ be the Fourier operator from $\C^N$ to $\C^N$, so that $\F^{-1}=\F^H$ is the inverse Fourier operator. $\mb{x}$ and $X$ are the vector form and operator form of the coefficient map, respectively. $\hat{\mb{x}}$ and $\hat{X}$ are the corresponding vector and operator in the frequency domain.
By definition, we have that $\hat{\mb{x}} = \F \mb{x}$.
We claim that
\begin{equation}
\label{eqn:operator-freq}
 \hat{X} = \F X \F^{H} \;.
\end{equation}
 To prove this, notice that
\[\hat{X}\hat{\mb{d}} = \F\big( \mb{x} \ast \mb{d} \big) = \F( X\mb{d} \big) = \F X \F^H \F \mb{d} = \F X \F^H \hat{\mb{d}} \;, \quad \forall \mb{d}\in\R^N \;. \]
Thus we have $\hat{X} = \F X \F^{H}$. With this equation, we have
\[
\begin{aligned}
\hat{X}^H(\hat{X} \hat{\mb{d}} - \hat{\mb{s}}) &= ( \F X \F^{H})^H ( \F X \F^{H} \F\mb{d} - \F\mb{s}) = (\F X^T \F^H) (\F X\mb{d} - \F \mb{s})\\& = \F \big( X^T(X\mb{d}-\mb{s}) \big) \;,
\end{aligned}
\]
which is exactly (\ref{eqn:equal_freq_spatial}).
\end{proof}

\section{Frequency-domain FISTA}
\label{sec:fista_freq}

To solve (\ref{eqn:d-update-mod}), we propose frequency-domain FISTA, Algorithm \ref{algo:fista}. It calculates the gradient in the frequency domain and do projection and extrapolation in the spatial domain. Mathematically speaking, (\ref{eqn:equal_freq_spatial}) illustrates that frequency-domain FISTA is actually equivalent with standard FISTA. However, calculating convolutional operator in the frequency domain reduces computing time. Thus, our algorithm is faster.

\begin{algorithm2e}[h]
\SetKwInOut{initial}{Initialize}
\KwIn{Hessian matrix $\hat{A}^{(t)}_{\mathrm{mod}}$ and vector $\hat{\mb{b}}^{(t)}_{\mathrm{mod}}$.\\
Dictionary of last iterate: $\mb{d}^{(t-1)}$.}
\initial{Let $\mb{g}^0 = \mb{d}^{(t-1)}$  (warm start),
$\mb{g}^0_{\text{aux}} = \mb{g}^0$, $\gamma^0 = 1$.}
\For{$j = 0,1,2,\ldots$ until condition (\ref{eqn:stop_condition}) is satisfied} {
Compute DFT: $\hat{\mb{g}}^j_{\text{aux}}= \text{FFT}(\mb{g}^j_{\text{aux}})$. \\
Compute conjugate cogradient:  $\nabla \hat{\F}^{(t)}_{\mathrm{mod}} (\hat{\mb{g}}^{j}_{\text{aux}}) = \frac{1}{\Lambda^{(t)}} \big(\hat{A}^{(t)}_{\mathrm{mod}} \hat{\mb{g}}^{j}_{\text{aux}} - \hat{\mb{b}}^{(t)}_{\mathrm{mod}}\big)$ .\\
Compute the next iterate:
\vspace{-2mm}
\begin{equation}
\label{ista}
\mb{g}^{j+1} = \text{Proj}_{\text{C}_\text{PN}} \Big(\text{IFFT}\big(\hat{\mb{g}}^{j}_{\text{aux}} - \eta \nabla \hat{\F}^{(t)}_{\mathrm{mod}} (\hat{\mb{g}}^{j}_{\text{aux}})\big)\Big) \;.
\vspace{-2mm}
\end{equation}
Let $\gamma^{j+1} = \big(1+\sqrt{1+4(\gamma^j)^2}\big) / 2$, then compute the auxiliary variable:
\vspace{-2mm}
\begin{equation}
\mb{g}^{j+1}_{\text{aux}} = \mb{g}^{j+1} + \frac{\gamma^{j}-1}{\gamma^{j+1}} (\mb{g}^{j+1}-\mb{g}^j) \;.
\vspace{-4mm}
\end{equation}
}
\KwOut{$\mb{d}^{(t)} \leftarrow \mb{g}^J$, where $J$ is the last iterate.}
\caption{Frequency-domain FISTA for solving subproblem (\ref{eqn:d-update-mod})}\label{algo:fista}
\end{algorithm2e}

\section{Details of the assumptions}\label{app:assume-details}

\subsection{Description of Assumption \ref{assume:uniqueness}}

To represent Assumption \ref{assume:uniqueness} in a concise way, we use the  notation
\[
D \mb{x} = \sum_{m=1}^M \mb{d}_m \ast \mb{x}_{m} \approx \mb{s} \;,
\]
where $\mb{x}\in\R^{MN}$, $\mb{s}\in\R^{N}$, $D: \R^{MN}\to\R^{N}$ is the convolutional dictionary. Then CBPDN problem (\ref{eq:cbpdn}) could be written as
\begin{equation}
\label{eqn:cbpdn_D}
\min_{\mb{x}\in\R^{MN}} (1/2)\norm{D\mb{x}-\mb{s}}_2^2 + \lambda \norm{\mb{x}}_1 \;.
\end{equation}
The coefficient map $\mb{x}$ is usually sparse, and $\Lambda$ is the set of indices of non-zero elements in $\mb{x}$. Then, we have
$D\mb{x} = D_{\Lambda}\mb{x}_{\Lambda}$.
By the results in \cite{fuchs2005recovery}, problem (\ref{eqn:cbpdn_D}) has the unique solution if $D_{\Lambda}^TD_{\Lambda}$ is invertible\footnote{Although \cite{fuchs2005recovery} only studies standard sparse coding, the uniqueness condition can be applied to the convolutional case because the only condition in their proof is ``for a  convex function $f(x)$ on $\R^n$, $x$ a minimum if and only if $0\in \partial f(x)$''. The only assumption is the convexity of the function, with no assumptions on the signals and dictionaries. Thus, large signals and convolutional dictionaries as in our case are consistent with the condition in \cite{fuchs2005recovery}.}, and its unique solution satisfies
\begin{equation}
\label{eqn:uniquesln}
\mb{x}^*_{\Lambda} = (D_{\Lambda}^TD_{\Lambda})^{-1}(D_{\Lambda}^T\mb{s} - \lambda \text{sign}(\mb{x}^*_{\Lambda})) \;.
\end{equation}
Specifically, Assumption \ref{assume:uniqueness} is: \emph{for all signals $\mb{s}$ and dictionaries $\mb{d}$, the smallest singular value of $D_{\Lambda}^TD_{\Lambda}$ is lower bounded by a positive number, i.e.}
\begin{equation}
\label{eqn:unique}
\sigma_{\text{min}}(D_{\Lambda}^TD_{\Lambda}) \geq \kappa \;.
\end{equation}

Except for condition (\ref{eqn:unique}), other types of uniqueness conditions of CSC are studied in recent works~\cite{papyan2017convolutional2, papyan2017working, sulam2017multi}.

\subsection{Description of Assumption \ref{assume:surro}}

Specifically, Assumption \ref{assume:surro} is, \emph{the surrogate functions $\F^{(t)}_{\mathrm{mod}}(\mb{d})$ are uniformly strongly convex, i.e.}
\begin{equation}
\label{eqn:strong-convex}
\langle  \nabla \F^{(t)}_{\mathrm{mod}}(\mb{d}) - \nabla \F^{(t)}_{\mathrm{mod}}(\tilde{\mb{d}})  , \mb{d} - \tilde{\mb{d}}\rangle \geq \mu \normsz[\big]{ \mb{d} - \tilde{\mb{d}} }^2 \;,
\end{equation}
\emph{for all $t, \mb{d},\tilde{\mb{d}}$, for some $\mu > 0$.}

\section{Proofs of propositions and the theorem}

Before proving propositions, we introduce a useful lemma.

\begin{lemma}[Uniform smoothness of surrogate functions]
\label{lemma:lipschitz} Under Assumptions \ref{assume:bdd_signal} and \ref{assume:uniqueness}, we have $f^{(t)}$ (\ref{eqn:loss_surrogate}) and $\F^{(t)}_{\mathrm{mod}} (\ref{eqn:F}) $ are uniformly $L$-smooth, i.e.
\begin{equation}
\label{eqn:lipschitz}
\begin{aligned}
\normsz[\big]{\nabla f^{(t)}(\mb{d}) - \nabla f^{(t)}(\tilde{\mb{d}})} &\leq L_f \normsz[\big]{\mb{d} - \tilde{\mb{d}} }\\
\normsz[\big]{\nabla \F^{(t)}_{\mathrm{mod}}(\mb{d}) - \nabla \F^{(t)}_{\mathrm{mod}}(\tilde{\mb{d}})} &\leq L_{\F} \normsz[\big]{\mb{d} - \tilde{\mb{d}} } \;,
\end{aligned}
\end{equation}
for all $t, \mb{d}, \tilde{\mb{d}}$, for some constants $L_f > 0,  L_{\F} > 0$.
\end{lemma}
\begin{proof}
First, we consider a single surrogate function:
\[
\normsz[\big]{\nabla f^{(t)}(\mb{d}) - \nabla f^{(t)}(\tilde{\mb{d}})} = \normsz[\big]{(X^{(t)})^T(X^{(t)})(\mb{d}-\tilde{\mb{d}})} \;.
\]
By $\mb{d} \in \text{C}$ (the compact support of $\mb{d}$), Assumption \ref{assume:bdd_signal} (the compact support of $\mb{s}$), and equation (\ref{eqn:uniquesln}) (regularity of convolutional sparse coding), we have $\mb{x}^{(t)}$ is uniformly bounded. Therefore, $X^{(t)}$, the operator form of $\mb{x}^{(t)}$, is also uniformly bounded:
\begin{equation}
\label{eqn:x_bdd}
\normsz[\big]{X^{(t)}} \leq M,
\end{equation}
for all $t$, for some $M>0$, which is independent of $t$.\\
By (\ref{eqn:surrogate_mod_explicit}), we have
\[\begin{aligned}
  \normsz[\Big]{\nabla \F^{(t)}_{\mathrm{mod}}(\mb{d}) - \nabla \F^{(t)}_{\mathrm{mod}}(\tilde{\mb{d}})}
  = &\normsz[\bigg]{\frac{1}{\Lambda^{(t)}}\sum_{\tau=1}^t(\tau/t)^p(X^{(\tau)})^T(X^{(\tau)})(\mb{d}-\tilde{\mb{d}})} \\
  \leq & \frac{1}{\Lambda^{(t)}}\sum_{\tau=1}^t(\tau/t)^p \normsz[\Big]{(X^{(\tau)})^T(X^{(\tau)})(\mb{d}-\tilde{\mb{d}})} \;,
\end{aligned}\]
which, together with (\ref{eqn:x_bdd}), implies (\ref{eqn:lipschitz}).
\end{proof}

\subsection{Proof of Proposition \ref{lemma:fista}}
\label{app:lemma-fista}

Given the strong-convexity (\ref{eqn:strong-convex}) and smoothness (\ref{eqn:lipschitz}) of the surrogate function, we start to prove Proposition \ref{lemma:fista}.
\begin{proof} To prove (\ref{eqn:fpr_bdd}), we consider a more general case. Let $\mb{g}^*$ be the minimizer of the following subproblem:
\[
\mb{g}^* = \argmin_{\mb{d}}\F(\mb{d}) + \iota_{\text{C}}(\mb{d}) \;,
\]
where $\F$ is $\mu$-strongly convex and $L$-smooth. Moreover, $\mb{g}^{j}$ and $\mb{g}_{\text{aux}}^{j}$ are the  iterates generated in Algorithm \ref{algo:fista}, and $j$ is the loop index. Then, we want to show that
\begin{equation}
\label{eqn:fpr_bdd-proof}
\norm{\mb{g}^{j+1}-\mb{g}^*} \leq C \fpr^{(t)}(\mb{g}^{j}_{\text{aux}}) \;, \quad \forall j \geq 0 \;.
\end{equation}

By (\ref{eqn:equal_freq_spatial}), it is enough to prove the above for the spatial-domain FISTA. By strong convexity and smoothness of $\F$, we obtain
\[
\begin{aligned}
&\normsz{\mb{g}^{j+1} - \mb{g}^*}^2\\
=& \normsz[\Big]{\text{Proj}(\mb{g}_{\text{aux}}^j - \eta \nabla \F(\mb{g}_{\text{aux}}^j)) - \text{Proj}(\mb{g}^* - \eta \nabla \F(\mb{g}^*)) }^2\\
\leq & \normsz[\Big]{\mb{g}_{\text{aux}}^j - \eta \nabla \F(\mb{g}_{\text{aux}}^j) - \mb{g}^* - \eta \nabla \F(\mb{g}^*) }^2\\
=&  \normsz[\Big]{\mb{g}_{\text{aux}}^j - \mb{g}^* - \eta \big( \nabla \F(\mb{g}_{\text{aux}}^j) - \nabla \F(\mb{g}^*) \big)}^2\\
=& \|\mb{g}_{\text{aux}}^j - \mb{g}^*\|^2 - 2\eta\Big\langle  \mb{g}_{\text{aux}}^j - \mb{g}^*, \nabla \F(\mb{g}_{\text{aux}}^j) - \nabla \F(\mb{g}^*)  \Big\rangle + \eta^2 \normsz[\big]{\nabla \F(\mb{g}_{\text{aux}}^j) - \nabla \F(\mb{g}^*)}^2\\
\leq & (1-2\mu\eta+ \eta^2L^2)\|\mb{g}_{\text{aux}}^j - \mb{g}^*\|^2 \;.
\end{aligned}
\]
Combining the above inequality and the definition of FPR (\ref{eq:fpr}), we have
\[
\begin{aligned}
\fpr(\mb{g}^{j}_{\text{aux}}) =& \normsz[\Big]{ \mb{g}^{j}_{\text{aux}} - \text{Proj}\big(\mb{g}^{j}_{\text{aux}} - \eta\nabla \F(\mb{g}^{j}_{\text{aux}})\big)}\\
=& \normsz{ \mb{g}^{j}_{\text{aux}} - \mb{g}^{(j+1)}}\\
=& \normsz{ \mb{g}^{j}_{\text{aux}} - \mb{g}^* - (\mb{g}^{(j+1)} - \mb{g}^* )}\\
\geq & \|\mb{g}^{j}_{\text{aux}} - \mb{g}^*\| - \|\mb{g}^{(j+1)} - \mb{g}^*\|\\
\geq & \Big(1-\sqrt{1-2\mu\eta + \eta^2L^2}\Big) \|\mb{g}^{j}_{\text{aux}} - \mb{g}^*\|\\
\geq & \frac{1-\sqrt{1-2\mu\eta + \eta^2L^2}}{\sqrt{1-2\mu\eta+ \eta^2L^2}}\|\mb{g}^{j+1} - \mb{g}^*\| \;.
\end{aligned}
\]
Let the step size be small enough $\eta \leq \min{(\mu/L^2,1/\mu)}$, we have $0\leq1-2\mu\eta+ \eta^2L^2\leq1$, which implies (\ref{eqn:fpr_bdd-proof}). Combining (\ref{eqn:fpr_bdd-proof}) and (\ref{eqn:stop_condition}), we get (\ref{eqn:fpr_bdd}).
\end{proof}

\subsection{Proof of Proposition \ref{lemma:d_residual}}
\label{app:lemma-d-residual}

\begin{proof} Recall $(\mb{d}^*)^{(t)}$ (\ref{eqn:exact-t}) is the ``exact solution'' of the $t^{\text{th}}$ iterate, and $\mb{d}^{(t)}$ is the ``inexact solution'' of the $t^{\text{th}}$ iterate (i.e. the approximated solution obtained by stopping condition (\ref{eqn:stop_condition})). Then, by the strong convexity of $\F^{(t)}_{\mathrm{mod}}$, we have
\[
\begin{aligned}
&\F^{(t)}_{\mathrm{mod}}(\mb{d}^{(t+1)}) - \F^{(t)}_{\mathrm{mod}}(\mb{d}^{(t)})\\
=& \F^{(t)}_{\mathrm{mod}}(\mb{d}^{(t+1)}) - \F^{(t)}_{\mathrm{mod}}((\mb{d}^*)^{(t)}) - \Big(\F^{(t)}_{\mathrm{mod}}(\mb{d}^{(t)}) - \F^{(t)}_{\mathrm{mod}}((\mb{d}^*)^{(t)})\Big)\\
\geq & \mu \|\mb{d}^{(t+1)} - (\mb{d}^*)^{(t)}\|^2 - L \| \mb{d}^{(t)} - (\mb{d}^*)^{(t)} \|^2\\
\geq & \mu \Big(\|\mb{d}^{(t+1)} - \mb{d}^{(t)}\| - \|\mb{d}^{(t)} - (\mb{d}^*)^{(t)}\|\Big)^2- L \| \mb{d}^{(t)} - (\mb{d}^*)^{(t)} \|^2 \;.
\end{aligned}
\]
Let $r^{(t)} = \|\mb{d}^{(t+1)} - \mb{d}^{(t)}\|$. If $r^{(t)}\leq C/t$, Proposition \ref{lemma:d_residual} is directly proved. Otherwise,  Proposition \ref{lemma:fista} (\ref{eqn:fpr_bdd}) implies $r^{(t)} - \|\mb{d}^{(t)} - (\mb{d}^*)^{(t)}\| \geq r^{(t)} - C/t \geq 0$ and
\begin{equation}
\label{eq:first_bd}
\F^{(t)}_{\mathrm{mod}}(\mb{d}^{(t+1)}) - \F^{(t)}_{\mathrm{mod}}(\mb{d}^{(t)}) \geq \mu\Big(r^{(t)} - \frac{C }{t}\Big)^2 - \frac{LC^2}{t^2} \;.
\end{equation}

On the other hand,
\[
\begin{aligned}
\F^{(t)}_{\mathrm{mod}}(\mb{d}^{(t+1)}) - \F^{(t)}_{\mathrm{mod}}(\mb{d}^{(t)}) =& \underbrace{\F^{(t)}_{\mathrm{mod}}(\mb{d}^{(t+1)}) - \F^{(t+1)}_{\mathrm{mod}}(\mb{d}^{(t+1)})}_{\T_1}\\
+\underbrace{ \F^{(t+1)}_{\mathrm{mod}}(\mb{d}^{(t+1)}) - \F^{(t+1)}_{\mathrm{mod}}(\mb{d}^{(t)})}_{\T_2} &+ \underbrace{\F^{(t+1)}_{\mathrm{mod}}(\mb{d}^{(t)}) - \F^{(t)}_{\mathrm{mod}}(\mb{d}^{(t)})}_{\T_3} \;,
\end{aligned}
\]
Now we will give the upper bounds of $\T_1,\T_2,\T_3$. Given the smoothness of $\F_{\mathrm{mod}}^{(t)} $(\ref{eqn:lipschitz}) and $(\mb{d}^*)^{(t+1)}$ being the minimizer of $\F_{\mathrm{mod}}^{(t)}$, we have an upper bound of $\T_2$:
\begin{equation}
\label{eqn:ft1ft}
\begin{aligned}
\T_2 =& \F^{(t+1)}_{\mathrm{mod}}(\mb{d}^{(t+1)}) - \F^{(t+1)}_{\mathrm{mod}}(\mb{d}^{(t)})\\
=&\F^{(t+1)}_{\mathrm{mod}}(\mb{d}^{(t+1)}) - \F^{(t+1)}_{\mathrm{mod}}((\mb{d}^*)^{(t+1)}) - \Big(\F^{(t+1)}_{\mathrm{mod}}(\mb{d}^{(t)}) - \F^{(t+1)}_{\mathrm{mod}}((\mb{d}^*)^{(t+1)})\Big)\\
\leq& L\|\mb{d}^{(t+1)} - (\mb{d}^*)^{(t+1)}\|^2 - 0 \leq \frac{LC^2}{t^2} \;.
\end{aligned}
\end{equation}
Based on (\ref{eqn:surrogate_mod_explicit}), the gradient of $\F^{(t)}_{\mathrm{mod}}-\F^{(t+1)}_{\mathrm{mod}}$ is bounded by
\[
\begin{aligned}
&\norm{\nabla \F^{(t)}_{\mathrm{mod}}(\mb{d}) - \nabla \F^{(t+1)}_{\mathrm{mod}}(\mb{d}) }\\
=&\norm{\nabla \F^{(t)}_{\mathrm{mod}}(\mb{d}) - \frac{\alpha^{(t+1)}\Lambda^{(t)}}{\Lambda^{(t+1)}} \nabla \F^{(t)}_{\mathrm{mod}}(\mb{d}) - \frac{1}{\Lambda^{(t+1)}}\nabla f^{(t+1)}(\mb{d})}\\
\leq&\frac{1}{\Lambda^{(t+1)}} \norm{\nabla \F^{(t)}_{\mathrm{mod}}(\mb{d})} +  \frac{1}{\Lambda^{(t+1)}}\norm{\nabla f^{(t+1)}(\mb{d})}
\leq C_0 /(\Lambda^{(t+1)}) \leq C_1/t \;,
\end{aligned}
\]
for some constant $C_1>0$. The second inequality follows from $\mb{d} \in \text{C}$ (the compact support of $\mb{d}$), Assumption \ref{assume:bdd_signal} (the compact support of $\mb{s}$), and equation (\ref{eqn:x_bdd}) (boundedness of $X$). The last inequality is derived by the follows:
\[
\frac{1}{\Lambda^{(t+1)}} = \frac{(t+1)^p}{\sum_{\tau=1}^{(t+1)} \tau^p}\leq \frac{(t+1)^p}{\int_0^{(t+1)}\tau^p \mathrm{d}\tau } = \frac{p}{t+1} \;.
\]
Then, $\F^{(t)}_{\mathrm{mod}}-\F^{(t+1)}_{\mathrm{mod}}$ is a Lipschitz continuous function with $L = C_1 /t$, which implies
\[
\T_1+\T_3 \leq \frac{C_1}{t}r^{(t)} \;.
\]
Therefore,
\begin{equation}
\label{eq:second_bd}
\F^{(t)}_{\mathrm{mod}}(\mb{d}^{(t+1)}) - \F^{(t)}_{\mathrm{mod}}(\mb{d}^{(t)}) \leq \frac{C_1}{t}r^{(t)} + \frac{LC^2}{t^2} \;.
\end{equation}

Combining (\ref{eq:first_bd}) and (\ref{eq:second_bd}), we have
\[
\mu\Big(r^{(t)} - \frac{C }{t}\Big)^2 - \frac{LC^2}{t^2} \leq \frac{C_1}{t}r^{(t)} + \frac{LC^2}{t^2} \;,
\]
which implies
\[
(r^{(t)})^2 - \frac{2C+C_1}{t}r^{(t)}\leq \frac{2LC^2}{\mu t^2} \;.
\]
This can be written more neatly as
\[(r^{(t)})^2 - 2\frac{C_2}{t}r^{(t)} \leq \frac{C_3}{t^2} \;, \quad \text{for some $C_2>0,C_3>0$} \;.
\]
Finally, $r^{(t)}$ is bounded by $r^{(t)}\leq (C_2+\sqrt{C_2^2+C_3})/t$. (\ref{eqn:d_residual}) is proved.
\end{proof}

\subsection{Proof of Proposition \ref{lemma:weight_clt}}
\label{app:weight_clt}

\begin{proof}
Define a sequence of random variables $Y_{i} = i^p Z_i$. Their expectations and variances are $\mu_i = i^p \mu$ and  $\sigma^2_i = i^{2p} \sigma^2$, respectively. Now we apply the Lyapunov central limit theorem on the stochastic sequence $\{Y_i\}$. First, we check the Lyapunov condition~\cite{billingsley2008probability}.
Let
\[
s_n^2 = \sum_{i=1}^n \sigma^2_i = \sum_{i=1}^n i^{2p}\sigma^2 = \Theta (n^{2p+1}) \;,
\]
then we have
\begin{equation}
\label{eqn:lyapunov} \frac{1}{s_n^{2+\delta}}\sum_{i=1}^n \E\Big[|Y_i - \mu_i|^{2+\delta}\Big]\leq \frac{1}{s_n^{2+\delta}}\sum_{i=1}^n (i^p\sigma)^{2+\delta} = \co\left( \frac{n^{2p + 1 + \delta p}}{n^{2p+1+\delta p + \delta/2}} \right) = \co (n^{-\delta/2}) \;.
\end{equation}
The Lyapunov condition is satisfied, so, by the Lyapunov central limit theorem, we have $\frac{1}{s_n}\sum_{i=1}^n (Y_i - \mu_i ) \overset{\text{d}}{\to} N(0,1)$. Furthermore, the definition of $\hat{Z}^n_{\mathrm{mod}}$ indicates
\[
\begin{aligned}
\frac{1}{s_n}\sum_{i=1}^n (Y_i - \mu_i ) =& \; \frac{1}{s_n}\sum_{i=1}^n i^p(Z_i - \mu ) = \frac{\sum_{i=1}^ni^p}{\sqrt{\sum_{i=1}^ni^{2p}}\sigma}\bigg(\frac{1}{\sum_{i=1}^ni^p}\sum_{i=1}^n i^p(Z_i - \mu ) \bigg)\\ =& \;\frac{\sum_{i=1}^ni^p}{\sqrt{\sum_{i=1}^ni^{2p}}\sigma} (\hat{Z}^n_{\mathrm{mod}} - \mu) \;.
\end{aligned}
\]
Given the following inequalities:
\[
\sum_{i=1}^ni^p < \int^{n+1}_1 s^p \text{d}s < \frac{1}{p+1}(n+1)^{p+1} \;, \quad
\sum_{i=1}^ni^p > \int^{n}_0 s^p \text{d}s = \frac{1}{p+1}(n)^{p+1} \;,
\]
we have
\[
\begin{aligned}\frac{1}{s_n}\sum_{i=1}^n (Y_i - \mu_i ) \leq \Big(1+\frac{1}{n}\Big)^{p+1} \frac{1}{\sigma}\frac{\sqrt{2p+1}}{p+1}\sqrt{n} (\hat{Z}^n_{\mathrm{mod}} - \mu) \;,\\\frac{1}{s_n}\sum_{i=1}^n (Y_i - \mu_i ) \geq \Big(1+\frac{1}{n}\Big)^{-(p+1)} \frac{1}{\sigma}\frac{\sqrt{2p+1}}{p+1}\sqrt{n} (\hat{Z}^n_{\mathrm{mod}} - \mu) \;.
\end{aligned}
\]
Then (\ref{eqn:weight_clt}) is obtained by $\frac{1}{s_n}\sum_{i=1}^n (Y_i - \mu_i ) \overset{\text{d}}{\to} N(0,1)$ and $(1+1/n)\to 1$.

The formula $\text{Var}(X) = \E X^2 - (\E X)^2 \geq0$ implies
\[
\bigg(\E\Big[\sqrt{n}\big|\hat{Z}^n_{\mathrm{mod}} - \mu\big|\Big]\bigg)^2 \leq \E\Big[n\big|\hat{Z}^n_{\mathrm{mod}} - \mu\big|^2\Big] \;.
\]
By the independence of different $Z_i$, we have
\[
\begin{aligned}
\E\Big[n\big|\hat{Z}^n_{\mathrm{mod}} - \mu\big|^2\Big] =& \frac{n}{(\sum_{i=1}^ni^p)^2} \sum_{i=1}^n \E\Big[ i^{2p}\big|Z_i - \mu\big|^2 \Big] \leq\frac{(p+1)^2}{2p+1} B^2 \;,
\end{aligned}
\]
where $B$ is the upper bound of $Z_i$ as $Z_i$ is compact supported. (\ref{eqn:weight_clt_bdd}) is proved.
\end{proof}

\subsection{Proof of Proposition \ref{lemma:weight}}
\label{app:weight-clt-f}

\begin{proof}
First, we fix $\mb{d}\in\text{C}$.
Let $i\to\tau, n\to t, Z_i \to f(\mb{d};\mb{s}^{(\tau)})$, then, by Proposition \ref{lemma:weight_clt}, we have
\[
\E\Big[\sqrt{t}\big|F(\mb{d}) - F^{(t)}_{\mathrm{mod}}(\mb{d})\big|\Big] \leq \frac{p+1}{\sqrt{2p+1}}B \;, \quad \forall t \in \{1,2,\cdots\}
\]
for some $B>0$, for fixed $\mb{d}$. Since $F$ and $F^{(t)}_{\mathrm{mod}}$ are continuously differentiable and have uniformly bounded derivatives (\ref{eqn:x_bdd}), we have $\E\Big[\sqrt{t}\big|F(\mb{d}) - F^{(t)}_{\mathrm{mod}}(\mb{d})\big|\Big]$ is uniformly continuous w.r.t $\mb{d}$ on a compact set $C$. Thus, the boundedness of  $\E\Big[\sqrt{t}\big|F(\mb{d}) - F^{(t)}_{\mathrm{mod}}(\mb{d})\big|\Big]$ on each $\mb{d}$ implies the boundedness for all $\mb{d}\in \text{C}$. Inequality (\ref{eqn:weight-clt-f}) is proved. Taking $p\to0$, we have (\ref{eqn:donsker}).
\end{proof}

\subsection{Proof of Theorem \ref{prop:main}}\label{app:mainprop}
\begin{proof}Let $u^{(t)} = \F^{(t)}_{\mathrm{mod}}(\mb{d}^{(t)})$. Inspired by the proof of Proposition 3 in \cite{mairal-2010-online}, we will show that $u^{(t)}$ is a ``quasi-martingale'' \cite{fisk1965quasi}.
\[
\begin{aligned}
&u^{(t+1)} - u^{(t)}\\
= &\F^{(t+1)}_{\mathrm{mod}}(\mb{d}^{(t+1)}) - \F^{(t)}_{\mathrm{mod}}(\mb{d}^{(t)})\\
 =& \underbrace{\F^{(t+1)}_{\mathrm{mod}}(\mb{d}^{(t+1)}) - \F^{(t+1)}_{\mathrm{mod}}(\mb{d}^{(t)})}_{\T_2} + \underbrace{\F^{(t+1)}_{\mathrm{mod}}(\mb{d}^{(t)}) -  \F^{(t)}_{\mathrm{mod}}(\mb{d}^{(t)})}_{\T_4} \;.
\end{aligned}
\]
The bound of $\T_2$ is given by (\ref{eqn:ft1ft}).  Furthermore, definition (\ref{eqn:loss_surrogate}) tells us $f^{(t+1)}(\mb{d}^{(t)}) = f(\mb{d}^{(t)};\mb{s}^{(t+1)})$, which implies
\[
\begin{aligned}
\T_4 =& \F^{(t+1)}_{\mathrm{mod}}(\mb{d}^{(t)}) -  \F^{(t)}_{\mathrm{mod}}(\mb{d}^{(t)})\\
=& \bigg(\frac{1}{\Lambda^{(t+1)}}f(\mb{d}^{(t)};\mb{s}^{(t+1)}) + \frac{\alpha^{(t+1)}\Lambda^{(t)}}{\Lambda^{(t+1)}} \F^{(t)}_{\mathrm{mod}}(\mb{d}^{(t)}) \bigg) -  \F^{(t)}_{\mathrm{mod}}(\mb{d}^{(t)})\\
=&\frac{f(\mb{d}^{(t)};\mb{s}^{(t+1)}) - F^{(t)}_{\mathrm{mod}}(\mb{d}^{(t)})}{\Lambda^{(t+1)}} + \frac{F^{(t)}_{\mathrm{mod}}(\mb{d}^{(t)}) - \F^{(t)}_{\mathrm{mod}}(\mb{d}^{(t)})}{\Lambda^{(t+1)}} \;.
\end{aligned}
\]
By the definitions of $f$ (\ref{eqn:cbpdn}) and $F$ (\ref{eqn:F}), we have $F^{(t)}_{\mathrm{mod}}(\mb{d}^{(t)}) \leq \F^{(t)}_{\mathrm{mod}}(\mb{d}^{(t)})$. Define $\G^t$ as all the previous information: $\G^t \triangleq \{ \mb{x}^{(\tau)}, \mb{s}^{(\tau)}, \mb{d}^{(\tau)} \}_{\tau=1}^t$. Thus, taking conditional expectation, we obtain
\[
\E[\T_4|\G^t] \leq \frac{1}{\Lambda^{(t+1)}}\Big(   \E[ f(\mb{d}^{(t)};\mb{s}^{(t+1)}) |\G^t]  - F^{(t)}_{\mathrm{mod}}(\mb{d}^{(t)})\Big)
= \frac{1}{\Lambda^{(t+1)}}\Big(   F(\mb{d}^{(t)})  - F^{(t)}_{\mathrm{mod}}(\mb{d}^{(t)})\Big) \;.
\]
Therefore, the positive part of $\E[\T_4|\G^t]$ is bounded by
\[
\E[\T_4|\G^t]^{+} \leq \frac{1}{\Lambda^{(t+1)}}\|   F  - F^{(t)}_{\mathrm{mod}}\|_{\infty} = \co\left(\frac{1}{t^{3/2}}\right)\;,
 \]
where the second inequality follows from (\ref{eqn:weight-clt-f}). Given the bound of $\T_2$ (\ref{eqn:ft1ft}) and $\T_4$, we have
\[
\sum_{t=1}^\infty \E\big[\E[u^{(t+1)}-u^{(t)}|\G^t]^{+}\big] \leq \sum_{t=1}^\infty \left( \co\left(\frac{1}{t^{3/2}}\right) + \co\left(\frac{1}{t^2}\right) \right)< + \infty \;,
\]
which implies that $u^{(t+1)}$ generated by Algorithm \ref{algo:surro-splitting}
is a quasi-martingale. Thus, by results in \cite[Sec. 4.4]{bottou1998online} or \cite[Theorem 6]{mairal-2010-online}, we have $u^{(t)}$ converges almost surely. \\ For the proofs of \ref{thm:2}, \ref{thm:3} and \ref{thm:4}, using the results in Proposition \ref{lemma:d_residual}, \ref{lemma:weight} in this paper, following the same proof line of Proposition 3 and 4 in \cite{mairal-2010-online}, we can obtain the results in \ref{thm:2}, \ref{thm:3} and \ref{thm:4}.
\end{proof}

\bibliographystyle{siamplainmod}
\bibliography{olcdl}

\end{document}